\documentclass[runningheads]{llncs}

\usepackage[utf8]{inputenc}
\usepackage{graphicx}
\usepackage{amsmath}
\usepackage{amssymb}
\usepackage{algorithm}
\usepackage[noend]{algorithmic}

\usepackage{xcolor}

\usepackage{multicol}
\usepackage{multirow}

\newcommand{\lst}[1]{\left( #1\right)}
\newcommand{\set}[1]{\left\{#1\right\}}
\newcommand{\tuple}[1]{\left\langle #1 \right\rangle}

\newcommand{\dd }[1]{\mathcal{#1}}
\newcommand{\arc}[1]{\lst{#1}}
\newcommand{\rst}[1]{\underline{\dd{#1}}}
\newcommand{\rlx}[1]{\overline {\dd{#1}}}

\newcommand{\pb}[1]{\mathcal{#1}}

\newcommand{\R}{\mathbb{R}}

\newcommand{\lb}[1]{\underline{#1}}
\newcommand{\ub}[1]{\overline{#1}}

\newcommand{\alg}[1]{Alg.~\ref{#1}}

\usepackage{tikzit}
\usetikzlibrary{decorations.pathmorphing}
\usetikzlibrary{decorations.markings}
\usetikzlibrary{snakes}
\tikzstyle{Exact}=[shape=circle, fill={rgb,255: red,222; green,255; blue,253}, draw=black, minimum size=1.5em, tikzit fill={rgb,255: red,222; green,255; blue,253}, tikzit draw=black]
\tikzstyle{Cutset}=[shape=circle, fill={purple!30}, draw=black, minimum size=1.5em, tikzit fill={rgb,255: red,255; green,143; blue,145}, tikzit draw=black]
\tikzstyle{Inexact}=[shape=circle, fill={red!10}, draw=black, double, minimum size=1.5em, tikzit fill={rgb,255: red,222; green,255; blue,253}, tikzit draw=black]
\tikzstyle{Suppressed}=[shape=circle, fill={red!20}, draw=black, dashed, minimum size=1.5em, tikzit fill={rgb,255: red,255; green,143; blue,145}, tikzit draw=black]
\tikzstyle{FringeElem}=[minimum height=52pt, fill={rgb,255: red,222; green,255;
blue,253}, draw=black, shape=circle, tikzit fill=white, tikzit draw=black, tikzit shape=circle]
\tikzstyle{Decision}=[color=purple]
\tikzstyle{Weight}=[font={\itshape \bfseries \scriptsize}, color={rgb,255: red,0; green,128; blue,128}]

\tikzstyle{RedArrow}=[draw=red, ->]
\tikzstyle{BlackArrow}=[draw=black, ->]
\tikzstyle{BestPath}=[-, line width=1.1pt, tikzit draw=red]
\tikzstyle{Pruned}=[-, dashed]
\tikzstyle{RelaxedArc}=[-, double, tikzit draw={rgb,255: red,255; green,191; blue,191}]
\tikzstyle{NonSolution}=[-, tikzit draw=magenta, snake=snake, segment amplitude=.5mm]

\begin{document}
\title{Improving the filtering of Branch-And-Bound MDD solver (extended)}

\author{
    Xavier Gillard     \inst{1}\orcidID{0000-0002-4493-6041}  \and
    Vianney Coppé      \inst{1}\orcidID{0000-0001-5050-0001}  \and
    Pierre Schaus      \inst{1}\orcidID{0000-0002-3153-8941}  \and
    André Augusto Cire \inst{2}\orcidID{0000-0001-5993-4295}
}
\institute{
    Université Catholique de Louvain, BELGIUM\\
    University of Toronto Scarborough and Rotman School of Management, CANADA
    \email{\{xavier.gillard, pierre.schaus, vianney.coppe\}@uclouvain.be},
    \email{andre.cire@rotman.utoronto.ca}
}
\authorrunning{
    X. Gillard et al.
}

\maketitle
\begin{abstract}
    This paper presents and evaluates two pruning techniques to reinforce the
    efficiency of constraint optimization solvers based on multi-valued 
    decision-diagrams (MDD). It adopts the branch-and-bound framework proposed 
    by Bergman et al. in 2016 to solve dynamic programs to optimality. In 
    particular, our paper presents and evaluates the effectiveness of the 
    local-bound (LocB) and rough upper-bound pruning (RUB). 
    LocB is a new and effective rule that leverages the approximate 
    MDD structure to avoid the exploration of non-interesting nodes. 
    RUB is a rule to reduce the search space during the development 
    of bounded-width-MDDs.
    The experimental study we conducted on the Maximum Independent Set Problem 
    (MISP), Maximum Cut Problem (MCP), Maximum 2 Satisfiability (MAX2SAT) and 
    the Traveling Salesman Problem with Time Windows (TSPTW)
    shows evidence indicating that rough-upper-bound and local-bound pruning 
    have a high impact on optimization solvers based on branch-and-bound with 
    MDDs. In particular, it shows that RUB delivers excellent results but 
    requires some effort when defining the model. Also, it shows that LocB 
    provides a significant improvement automatically; without necessitating any 
    user-supplied information. Finally, it also shows that rough-upper-bound and 
    local-bound pruning are not mutually exclusive, and their combined benefit 
    supersedes the individual benefit of using each technique.
\end{abstract}

\section*{Introduction}
\emph{Multi-valued Decision Diagrams} (MDD) are a generalization of 
\emph{Binary Decision Diagrams} (BDD) which have long been used in the 
verification, e.g., for model checking purposes 
\cite{mcmillan:92:symbolic-model-checking}. Recently, these graphical 
models have drawn the attention of researchers from the CP and OR communities. 
The popularity of these decision diagrams (DD) stems from their ability to 
provide a compact representation of large solution spaces as in the case of the 
table constraint \cite{regin:15:mdd,helene:18:compact-mdd}. 
One of the 
research streams which emerged from this increased interest about MDDs is
\emph{decision-diagram-based optimization} (DDO) \cite{bergman:16:theoretical}. 
Its purpose is to efficiently solve combinatorial optimization problems by 
exploiting problem structure through DDs. 
This paper belongs to the DDO sub-field and intends 
to further improve the efficiency of DDO solvers through the introduction of two 
bounding techniques: local-bounds pruning (LocB) and rough-upper-bound pruning 
(RUB).

This paper starts by covering the necessary background on DDO. Then, it presents 
the  local-bound and rough-upper-bound pruning techniques in Sections 
\ref{sec:local-bounds} and \ref{sec:rub}. After that, it presents an 
experimental study which we conducted using `ddo' 
\cite{gillard:20:ddo}\footnote{\url{https://github.com/xgillard/ddo}}, our open 
source fast and generic MDD-based optimization library. This experimental study 
investigates the relevance of RUB and LocB through four disinct NP-hard 
problems: the Weighted 
Maximum Independent Set Problem (MISP), Maximum Cut Problem (MCP), Maximum 2 
Satisfiability Problem (MAX2SAT) and the Traveling Salesman Problem with Time 
Windows (TSPTW). Finally, section \ref{sec:related} discusses previous related 
work before drawing conclusions.

\section{Background}\label{sec:background}
The coming paragraphs give an overview of discrete optimization with decision 
diagrams. Most of the formalism presented here originates from 
\cite{bergman:16:branch-and-bound}. Still, we reproduce it here for the sake of 
self-containedness.
\paragraph{Discrete optimization.}\label{sec:discrete-optimization}
A discrete optimization problem is a constraint \emph{satisfaction} problem with 
an associated objective function to be maximized. The discrete optimization 
problem $\pb{P}$ is defined as $\max \set{f(x) \mid x \in D \wedge C(x)}$ where 
$C$ is a set of constraints, $x = \tuple{x_0, \dots, x_{n-1}}$ is an assignment 
of 
values to variables, each of which has an associated finite domain $D_i$ s.t. $D 
= D_0 \times \dots \times D_{n-1}$ from where the values are drawn. In that 
setup, 
the function $f: D \to \R$ is the objective to be maximized. 

Among the set of feasible solutions $Sol(\pb{P}) \subseteq D$ (i.e. satisfying 
all constraints in $C$), we denote the optimal solution by $x^*$. That is, $x^* 
\in Sol(\pb{P})$ and $\forall x \in Sol(\pb{P}) : f(x^*) \ge f(x)$.

\paragraph{Dynamic programming.}
Dynamic programming (DP) was introduced in the mid 50's by Bellman 
\cite{bellman:54}. This strategy is significantly popular and is at the heart of 
many classical algorithms (e.g., Dijkstra's 
algorithm \cite[p.658]{cormen:09:algorithms} or Bellman-Ford's 
\cite[p.651]{cormen:09:algorithms}).

Even though a dynamic program is often thought of in terms of recursion, it is 
also natural to consider it as a labeled transition system. In that case, the 
\emph{DP model} of a given discrete optimization problem $\pb{P}$ consists 
of:
\begin{itemize}
    \item a set of state-spaces $S_0, \dots, S_{n}$ among which one      
    distinguishes the \emph{initial state} $r$, the \emph{terminal 
        state} $t$ and the \emph{infeasible state} $\bot$.
    \item a set of transition functions $t_i : S_i \times D_i \to S_{i+1}$ for  
    $i = 0, \dots, n-1$ taking the system from one state $s^{i}$ to the next 
    state $s^{i+1}$ based on the value $d$ assigned to variable $x_i$ (or 
    to $\bot$ if assigning $x_i = d$ is infeasible). These functions 
    should never allow one to recover from infeasibility ($t_i(\bot, d) 
    = \bot$ for any $d \in D_i$).
    \item a set of transition cost functions $h_i: S_i \times D_i \to \R$   
    representing the immediate reward of assigning some value $d \in D_i$ to the 
    variable $x_i$ for $i = 0, \dots, n-1$.
    \item an initial value $v_r$.      
\end{itemize}

On that basis, the objective function $f(x)$ of $\pb{P}$ can be formulated as 
follows:

\begin{align*}
& \text{maximize~} f(x) = v_r + \sum_{i = 0}^{n-1} h_i(s^i, x_i) \\
& \text{subject to}\\
& s^{i+1} = t_i(s^i, x_i) \text{~for~} i = 0, \dots, n-1 ; x_i \in D_i \wedge 
C(x_i) \\
& s^i \in S_i \text{~for~} i = 0, \dots, n
\end{align*}
where $C(x_i)$ is a predicate that evaluates to $true$ when the partial 
assignment $\tuple{x_0, \dots ,x_i}$ does not violate any constraint in $C$.

The appeal of such a formulation stems from its simplicity and its 
expressiveness which allows it to effectively capture the problem structure. 
Moreover, this formulation naturally lends itself to a DD representation; in 
which case it represents an exact DD encoding the complete set $Sol(\pb{P})$. 

\subsection{Decision diagrams} \label{sec:dd}
Because DDO aims at solving constraint \emph{optimization} problems and not just 
constraint \emph{satisfaction} problems, it uses a particular DD flavor known as 
reduced weighted DD -- DD as of now. As initially posed by 
Hooker\cite{hooker:13:dd-dp}, DDs can be perceived as a compact representation 
of the search trees. This is achieved, in this context, by superimposing 
isomorphic subtrees.

To define our DD more formally, we will slightly adapt the notation from  
\cite{bergman:16:theoretical}. 
A DD $\dd{B}$ is a layered directed acyclic graph $\dd{B} = \tuple{n, U, A, l,  
    d, v, \sigma}$ where $n$ is the number of variables from the encoded 
    problem, 
$U$ is a set of nodes; each of which is associated to some state $\sigma(u)$.
The mapping $l: U \to \set{0 \dots n}$ partitions the nodes from $U$ in 
disjoint layers $L_0 \dots L_{n}$ s.t. $L_i = \set{u \in U: l(u) = i}$ and the 
states of all the nodes belonging to the same layer pertain to the same 
DP-state-space ($\forall u \in L_i : \sigma(u) \in S_i$ for $i = 0, \dots, 
n$). Also, it should be the case that no two distinct nodes of one same layer 
have the same state ($\forall u_1, u_2 \in 
L_i: u_1 \ne u_2 \implies \sigma(u_1) \ne \sigma(u_2)$, for $i = 0, \dots, 
n$).

The set $A \subseteq U \times U$ from our formal model is a set of directed 
arcs connecting the nodes from $U$. Each such arc $a = \arc{u_1, u_2}$ connects 
nodes from subsequent layers ($l(u_1) = l(u_2) - 1$) and should be regarded as 
the materialization of a branching decision about variable $x_{l(u_1)}$. This is 
why all arcs are annotated via the mappings $d: A \to D$ and $v: A \to \R$ 
which  respectively associate a decision and value (weight) with the given arc. 

\begin{example}
    An arc $a$ connecting nodes $u_1 \in L_3$ to $u_2 \in L_4$, annotated with 
    $d(a) = 6$ and $v(a) = 42$ should be understood as the assignment $x_3 = 6$ 
    performed from state $\sigma(u_1)$. It should also be understood that 
    $t_3(\sigma(u_1), 6) = \sigma(u_2)$ and the benefit of that assignment is 
    $v(a) = h_3(\sigma(u_1), 6) = 42$.
\end{example}

Because each $r$-$t$ path describes an  
assignment that satisfies $\pb{P}$, we will use $Sol(\dd{B})$ to denote the set 
of all the solutions encoded in the r-t paths of DD $\dd{B}$. Also, because 
unsatisfiability is irrecoverable, r-$\bot$ paths are typically omitted from  
DDs. It follows that a nice property from using a DD representation 
$\dd{B}$ for the DP formulation of a problem $\pb{P}$, is that finding $x^*$ is 
as simple as finding the longest r-t path in $\dd{B}$ (according to the relation 
$v$ on arcs).

\paragraph{Exact-MDD.} \label{sec:exact}
For a given problem $\pb{P}$, an exact MDD $\dd{B}$ is an MDD that exactly  
encodes the solution set $Sol(\dd{B}) = Sol(\pb{P})$ of the problem $\pb{P}$. In 
other words, not only do all r-t paths encode valid solutions of $\pb{P}$, but 
no feasible solution is present in $Sol(\pb{P})$ and not in $\dd{B}$.
An exact MDD for $\pb{P}$ can be compiled in a top-down fashion\footnote{
    An incremental refinement \emph{a.k.a. construction by separation}  
    procedure is detailed in \cite[pp. 51--52]{cire:14:decision-diagrams} but we 
    will not cover it here for the sake of conciseness.}. 
This naturally follows from the above definition. To that end, one simply 
proceeds by a repeated unrolling of the transition relations until all 
variables are assigned.

\subsection{Bounded-Size Approximations} \label{sec:approx}
In spite of the compactness of their encoding, the construction of DD suffers 
from a potentially exponential memory requirement in the worst 
case\footnote{
    Consequently, it also suffers from a potentially exponential time 
    requirement in the worst case. Indeed, time is constant in the final number 
    of nodes (unless the transition functions themselves are exponential in the 
    input).}. 
Thus, using DDs to exactly encode the solution space of a problem is often 
intractable. Therefore, one must resort to the use of \emph{bounded-size} 
approximation of the exact MDD. These are compiled generically by inserting a 
call to a width-bounding procedure to ensure that the width (the number  
$|L_i|$ of distinct nodes belonging to the $L_i$) of the current layer $L_i$ 
does not exceed a given bound $W$. Depending on the behavior of that procedure, 
one can either compile a  restricted-MDD (= an under-approximation) or a 
relaxed-MDD (= an over-approximation). 

\paragraph{Restricted-MDD: Under-approximation.} \label{sec:restricted}
A restricted-MDD provides an under-approximation of some exact-MDD. As such, all 
paths of a restricted-MDD encode valid solutions, but some solutions might be 
missing from the MDD. This is formally expressed as follows: given the DP 
formulation of a problem $\pb{P}$, $\dd{B}$ is a restricted-MDD iff $Sol(\dd{B}) 
\subseteq Sol(\pb{P})$.

To compile a restricted-MDD, it is sufficient to simply delete certain nodes 
from the current layer until its width fits within the specified bound $W$. To 
that end, the width-bounding procedure simply selects a subset of the nodes from 
$L_i$ which are heuristically assumed to have the less impact on the tightness 
of the bound. Various heuristics have been studied in the literature 
\cite{bergman:14:optimization-bounds}, and \emph{minLP} was shown to be the 
heuristic that works best in practice. This heuristic decides to select (hence 
remove) the nodes having the shortest longest path from the root.

\paragraph{Relaxed-MDD: Over-approximation.} \label{sec:relaxed}
A relaxed-MDD $\dd{B}$ provides a bounded-width over-approximation of some  
exact-MDD. As such, it may hold paths that are no solution to $\pb{P}$, the 
problem being solved. We have thus formally that $Sol(\dd{B}) \supseteq 
Sol(\pb{P})$.

Compiling a relaxed-MDD requires one to be able to \emph{merge} several nodes  
into an inexact one. To that end, we use two operators: 
\begin{itemize}
    \item $\oplus$ which yields a new node combining the states of a selection
    of nodes so as to over-approximate the states reachable in the selection.
    \item $\Gamma$ which is used to possibly relax the weight of arcs incident to
    the selected nodes.
\end{itemize}

These operators are used as follows. Similar to the restricted-MDDs case, the 
width-bounding procedure starts by heuristically selecting the least promising 
nodes and removing them from layer $L_i$. Then the states of these selected 
nodes are combined with one another so as to create a merged node $\mathcal{M} = 
\oplus(selection)$. After that, the inbound arcs incident to all selected nodes 
are $\Gamma$-relaxed and redirected towards $\mathcal{M}$. Finally, the 
result of the merger ($\mathcal{M}$) is added to the layer in place of the 
initial selection of nodes.

\paragraph{Summary.} 
Fig.~\ref{fig:mdd:exact-restricted-relaxed} summarizes the information from 
sections \ref{sec:dd} and \ref{sec:relaxed}. It displays the three MDDs 
corresponding to one same example problem having four variables. The exact MDD 
(a) encodes the complete solution set and, equivalently, the state space of the 
underlying DP encoding. One easily notices that the 
restricted DD (b) is an under approximation of (a) since it achieves its width 
boundedness by removing nodes d and e and their children (i, j). Among others, 
it follows that the solution $[ x_0= 0, x_1 = 0, x_2 = 0, x_3 = 0 ]$ is not 
represented in (b) even though it exists in (a). Conversely, the relaxed diagram 
(c) achieves a maximum layer with of 3 by merging nodes d, e and h into a new 
inexact node $\mathcal{M}$ and by relaxing all arcs entering one of the merged 
nodes. Because of this, (c) introduces solutions that do not exist in (a) as is 
for instance the case of the assignment $[ x_0= 0, x_1 = 0, x_2 = 3, x_3 = 1 ]$. 
Moreover, because the operators $\oplus$ and $\Gamma$ are correct\footnote{The 
    very definition of these operators is problem-specific. However, 
    \cite{hooker:17:job-sequencing} formally defines the conditions that are 
    necessary to correctness.}, the length of the longest path in (c) is an 
    upper 
bound on the optimal value of the objective function. Indeed, one can see that 
the length of the longest path in (a) (= the exact optimal solution) has a value 
of 25 while it amounts to 26 in (c).

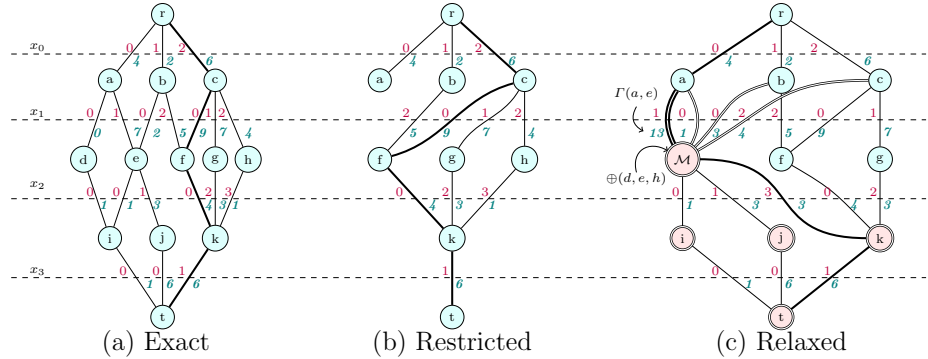
\begin{figure*}[h]
    \centering
    \scalebox{.7}{
        \scriptsize
        \begin{tikzpicture}
	\begin{pgfonlayer}{nodelayer}
		\node [style=Exact] (0) at (-7.25, 7) {r};
		\node [style=Exact] (1) at (-9.25, 4.5) {a};
		\node [style=Exact] (2) at (-7.25, 4.5) {b};
		\node [style=Exact] (3) at (-5.25, 4.5) {c};
		\node [style=Exact] (4) at (-10.25, 1.5) {d};
		\node [style=Exact] (5) at (-8.25, 1.5) {e};
		\node [style=Exact] (6) at (-6.5, 1.5) {f};
		\node [style=Exact] (7) at (-5.25, 1.5) {g};
		\node [style=Exact] (8) at (-4, 1.5) {h};
		\node [style=Exact] (9) at (-5.25, -1.5) {k};
		\node [style=Exact] (10) at (-7.25, -1.5) {j};
		\node [style=Exact] (11) at (-9.25, -1.5) {i};
		\node [style=Exact] (12) at (-7.25, -4.5) {t};
		\node [style=none] (14) at (-7.5, -5.5) {{\Large (a) Exact}};
		\node [style=Exact] (15) at (3.75, 7) {r};
		\node [style=Exact] (16) at (1, 4.5) {a};
		\node [style=Exact] (17) at (3.75, 4.5) {b};
		\node [style=Exact] (18) at (6.5, 4.5) {c};
		\node [style=Exact] (21) at (1, 1.5) {f};
		\node [style=Exact] (22) at (3.75, 1.5) {g};
		\node [style=Exact] (23) at (6.5, 1.5) {h};
		\node [style=Exact] (24) at (3.75, -1.5) {k};
		\node [style=Exact] (27) at (3.75, -4.5) {t};
		\node [style=none] (29) at (3.75, -5.5) {{\Large (b) Restricted}};
		\node [style=Exact] (30) at (16.25, 7) {r};
		\node [style=Exact] (31) at (12.5, 4.5) {a};
		\node [style=Exact] (32) at (16.25, 4.5) {b};
		\node [style=Exact] (33) at (20, 4.5) {c};
		\node [style=Exact] (36) at (16.25, 1.5) {f};
		\node [style=Exact] (37) at (20, 1.5) {g};
		\node [style=Inexact] (39) at (20, -1.5) {k};
		\node [style=Inexact] (40) at (16.25, -1.5) {j};
		\node [style=Inexact] (41) at (12.5, -1.5) {i};
		\node [style=Inexact] (42) at (16.25, -4.5) {t};
		\node [style=none] (43) at (16.25, -5.5) {{\Large (c) Relaxed}};
		\node [style=Inexact] (44) at (12.5, 1.5) {$\mathcal{M}$};
		\node [style=none] (45) at (10.75, 0.75) {$\oplus(d,e, h)$};
		\node [style=none] (46) at (10.75, 1) {};
		\node [style=none] (49) at (-13, 5.5) {};
		\node [style=none] (50) at (21.75, 5.5) {};
		\node [style=none] (51) at (-13, 3) {};
		\node [style=none] (52) at (21.75, 3) {};
		\node [style=none] (53) at (-13, 0) {};
		\node [style=none] (54) at (21.75, 0) {};
		\node [style=none] (55) at (-13, -3) {};
		\node [style=none] (56) at (21.75, -3) {};
		\node [style=none] (58) at (-12, 3.25) {$x_1$};
		\node [style=none] (59) at (-12, 0.5) {$x_2$};
		\node [style=none] (60) at (-12, -2.75) {$x_3$};
		\node [style=none] (61) at (-12, 5.75) {$x_0$};
		\node [style=Decision] (62) at (-8.5, 5.75) {0};
		\node [style=Decision] (63) at (-7.5, 5.75) {1};
		\node [style=Decision] (64) at (-6.5, 5.75) {2};
		\node [style=Decision] (65) at (-10, 3.25) {0};
		\node [style=Decision] (66) at (-9, 3.25) {1};
		\node [style=Decision] (67) at (-8, 3.25) {0};
		\node [style=Decision] (68) at (-7.25, 3.25) {2};
		\node [style=Decision] (69) at (-6, 3.25) {0};
		\node [style=Decision] (70) at (-5.5, 3.25) {1};
		\node [style=Decision] (71) at (-5, 3.25) {2};
		\node [style=Decision] (72) at (-10, 0.25) {0};
		\node [style=Decision] (73) at (-9, 0.25) {0};
		\node [style=Decision] (74) at (-8, 0.25) {1};
		\node [style=Decision] (75) at (-6.25, 0.25) {0};
		\node [style=Decision] (76) at (-5.5, 0.25) {2};
		\node [style=Decision] (77) at (-4.75, 0.25) {3};
		\node [style=Decision] (78) at (-8.75, -2.75) {0};
		\node [style=Decision] (79) at (-7.5, -2.75) {0};
		\node [style=Decision] (80) at (-6.5, -2.75) {1};
		\node [style=Weight] (81) at (-8.25, 5.25) {4};
		\node [style=Weight] (82) at (-7, 5.25) {2};
		\node [style=Weight] (83) at (-5.5, 5.25) {6};
		\node [style=Weight] (84) at (-9.75, 2.5) {0};
		\node [style=Weight] (85) at (-8.25, 2.5) {7};
		\node [style=Weight] (86) at (-7.5, 2.5) {2};
		\node [style=Weight] (87) at (-6.5, 2.5) {5};
		\node [style=Weight] (88) at (-5.75, 2.5) {9};
		\node [style=Weight] (89) at (-5, 2.5) {7};
		\node [style=Weight] (90) at (-4, 2.5) {4};
		\node [style=Weight] (91) at (-9.5, -0.25) {1};
		\node [style=Weight] (92) at (-8.5, -0.25) {1};
		\node [style=Weight] (93) at (-7.5, -0.25) {3};
		\node [style=Weight] (94) at (-5.5, -0.25) {4};
		\node [style=Weight] (95) at (-5, -0.25) {3};
		\node [style=Weight] (96) at (-4.5, -0.25) {1};
		\node [style=Weight] (97) at (-7.75, -3.25) {1};
		\node [style=Weight] (98) at (-7, -3.25) {6};
		\node [style=Weight] (99) at (-6, -3.25) {6};
		\node [style=Decision] (100) at (2, 5.75) {0};
		\node [style=Decision] (101) at (3.5, 5.75) {1};
		\node [style=Decision] (102) at (4.75, 5.75) {2};
		\node [style=Decision] (106) at (2, 3.25) {2};
		\node [style=Decision] (107) at (3.5, 3.25) {0};
		\node [style=Decision] (108) at (5, 3.25) {1};
		\node [style=Decision] (109) at (6.25, 3.25) {2};
		\node [style=Decision] (113) at (1.75, 0.25) {0};
		\node [style=Decision] (114) at (3.5, 0.25) {2};
		\node [style=Decision] (115) at (5, 0.25) {3};
		\node [style=Decision] (118) at (3.5, -2.75) {1};
		\node [style=Weight] (119) at (2.25, 5.25) {4};
		\node [style=Weight] (120) at (4, 5.25) {2};
		\node [style=Weight] (121) at (6, 5.25) {6};
		\node [style=Weight] (125) at (2.25, 2.5) {5};
		\node [style=Weight] (126) at (3.5, 2.5) {9};
		\node [style=Weight] (127) at (5, 2.5) {7};
		\node [style=Weight] (128) at (6.75, 2.5) {4};
		\node [style=Weight] (132) at (3, -0.25) {4};
		\node [style=Weight] (133) at (4, -0.25) {3};
		\node [style=Weight] (134) at (5.25, -0.25) {1};
		\node [style=Weight] (137) at (4, -3.25) {6};
		\node [style=Decision] (138) at (13.75, 5.75) {0};
		\node [style=Decision] (139) at (16, 5.75) {1};
		\node [style=Decision] (140) at (17.5, 5.75) {2};
		\node [style=Decision] (144) at (16, 3.25) {2};
		\node [style=Decision] (145) at (17.75, 3.25) {0};
		\node [style=Decision] (146) at (19.75, 3.25) {1};
		\node [style=Decision] (148) at (12.25, 0.25) {0};
		\node [style=Decision] (150) at (13.75, 0.25) {1};
		\node [style=Decision] (151) at (18, 0.25) {0};
		\node [style=Decision] (152) at (19.75, 0.25) {2};
		\node [style=Decision] (153) at (15.75, 0.25) {3};
		\node [style=Decision] (154) at (13.75, -2.75) {0};
		\node [style=Decision] (155) at (16, -2.75) {0};
		\node [style=Decision] (156) at (18, -2.75) {1};
		\node [style=Weight] (157) at (14.25, 5.25) {4};
		\node [style=Weight] (158) at (16.5, 5.25) {2};
		\node [style=Weight] (159) at (19.5, 5.25) {6};
		\node [style=Weight] (163) at (16.5, 2.5) {5};
		\node [style=Weight] (164) at (17.75, 2.5) {9};
		\node [style=Weight] (165) at (20.25, 2.5) {7};
		\node [style=Weight] (167) at (12.75, -0.25) {1};
		\node [style=Weight] (169) at (15.25, -0.25) {3};
		\node [style=Weight] (170) at (19, -0.25) {4};
		\node [style=Weight] (171) at (20.25, -0.25) {3};
		\node [style=Weight] (172) at (17, -0.25) {3};
		\node [style=Weight] (173) at (15, -3.25) {1};
		\node [style=Weight] (174) at (16.5, -3.25) {6};
		\node [style=Weight] (175) at (18.25, -3.25) {6};
		\node [style=Decision] (176) at (12.5, 3.25) {0};
		\node [style=Decision] (177) at (13.75, 3.25) {0};
		\node [style=Weight] (178) at (12.5, 2.5) {1};
		\node [style=Weight] (179) at (13.75, 2.5) {3};
		\node [style=none] (182) at (10.75, 4) {$\Gamma(a, e)$};
		\node [style=none] (183) at (10.75, 3.5) {};
		\node [style=Decision] (184) at (14.75, 3.25) {2};
		\node [style=Weight] (185) at (14.75, 2.5) {4};
		\node [style=Decision] (186) at (11.5, 3.25) {1};
		\node [style=Weight] (187) at (11.5, 2.5) {13};
	\end{pgfonlayer}
	\begin{pgfonlayer}{edgelayer}
		\draw [style=BestPath] (0) to (3);
		\draw (0) to (2);
		\draw (0) to (1);
		\draw (2) to (5);
		\draw (2) to (6);
		\draw [style=BestPath] (3) to (6);
		\draw (3) to (7);
		\draw (3) to (8);
		\draw (1) to (5);
		\draw (1) to (4);
		\draw [style=BestPath] (6) to (9);
		\draw (7) to (9);
		\draw (8) to (9);
		\draw (5) to (10);
		\draw (5) to (11);
		\draw (4) to (11);
		\draw (10) to (12);
		\draw [style=BestPath] (9) to (12);
		\draw [style=BestPath] (15) to (18);
		\draw (15) to (17);
		\draw (15) to (16);
		\draw (17) to (21);
		\draw [style=BestPath, in=15, out=-165] (18) to (21);
		\draw [in=60, out=-120] (18) to (22);
		\draw (18) to (23);
		\draw [style=BestPath] (21) to (24);
		\draw (22) to (24);
		\draw (23) to (24);
		\draw [style=BestPath] (24) to (27);
		\draw (11) to (12);
		\draw (30) to (33);
		\draw (30) to (32);
		\draw [style=BestPath] (30) to (31);
		\draw (32) to (36);
		\draw (33) to (36);
		\draw (33) to (37);
		\draw [in=135, out=-30, looseness=1.25] (36) to (39);
		\draw (37) to (39);
		\draw (40) to (42);
		\draw [style=BestPath] (39) to (42);
		\draw (41) to (42);
		\draw [style=RelaxedArc, in=45, out=-165, looseness=1.25] (32) to (44);
		\draw [style=RelaxedArc, style=BestPath, bend right] (31) to (44);
		\draw (44) to (41);
		\draw (44) to (40);
		\draw [style=BestPath, in=180, out=0, looseness=1.50] (44) to (39);
		\draw [style=BlackArrow, in=150, out=105, looseness=1.50] (46.center) to (44);
		\draw [style=Pruned] (49.center) to (50.center);
		\draw [style=Pruned] (51.center) to (52.center);
		\draw [style=Pruned] (53.center) to (54.center);
		\draw [style=Pruned] (55.center) to (56.center);
		\draw [style=RelaxedArc, in=30, out=-180, looseness=1.25] (33) to (44);
		\draw [style=RelaxedArc, in=60, out=-60] (31) to (44);
		\draw [style=BlackArrow, bend right=45, looseness=1.25] (183.center) to (187);
	\end{pgfonlayer}
\end{tikzpicture}
    }
    \caption{
        The exact (a), restricted (b) and relaxed (c) versions of an 
        MDD with four variables. The width of MDDs (b) and (c) have been 
        bounded to a maximum layer width of three. The decision labels of the 
        arcs are shown above the layers separation lines (dashed). The arc 
        weights are shown below the layer separation lines. The longest path of 
        each MDD is boldfaced. In (c), the node $\mathcal{M}$ is the result of 
        merging nodes d, e and h with the $\oplus$ operator. Arcs that 
        have been relaxed with the $\Gamma$ operator are pictured with a double 
        stroke. Note, because these arcs have been $\Gamma$-relaxed, their value 
        might be greater than that of corresponding arcs in (a), (b). 
        Similarly, all ``inexact'' nodes feature a double border.
    }
    \label{fig:mdd:exact-restricted-relaxed}
\end{figure*}

\subsection{The Dynamics of Branch-and-Bound with DDs}\label{sec:bab-dd}
Being able to derive good lower and upper  bounds for some optimization problem  
$\pb{P}$ is useful when the goal is to use these bounds to strengthen algorithms 
\cite{davarnia:18:approx-non-linear,tjandraatmadja:18:decision-diagrams,tjandraatmadja:19:cuts-from-relaxed}.
But it is not the only way these approximations can be used. A complete 
and efficient branch-and-bound algorithm relying on those approximations was 
proposed in \cite{bergman:16:branch-and-bound} which we hereby reproduce 
(\alg{alg:branch-and-bound}).

This algorithm works as follows: at start, the node $r$ is created for the 
initial state of the problem and placed onto the \emph{fringe} -- a global 
priority queue that tracks all nodes remaining to explore and orders them from 
the most to least promising. Then, a loop consumes the nodes from that fringe 
(line \ref{alg:branch-and-bound:1a}), one at a time and explores it until the 
complete state space has been exhausted. The \emph{exploration} of a node $u$ 
inside that loop proceeds as follows: first, one compiles a restricted DD 
$\rst{B}$ for the sub-problem rooted in $u$ (line \ref{alg:branch-and-bound:1b}).
Because all paths in a restricted DD are feasible solutions, when the lower 
bound $v^*(\rst{B})$ derived from the restricted DD $\rst{B}$ improves over the 
current best known solution $\lb{v}$; then the longest path of $\rst{B}$ (best 
sol. found in $\rst{B}$) and its length $v^*(\rst{B})$ are memorized (lines 
\ref{alg:branch-and-bound:improve-lb-start}-\ref{alg:branch-and-bound:improve-lb-end}).

In the event where $\rst{B}$ is exact (no restriction occurred during the 
compilation of $\rst{B}$), it covers the complete state space of the sub-problem 
rooted in $u$. Which means the processing of $u$ is complete and we may safely 
move to the next node. When this condition is not met, however, some additional 
effort is required.% (line \ref{alg:branch-and-bound:prune} and next).
In that case, a \emph{relaxed} DD $\rlx{B}$ is compiled from $u$(line 
\ref{alg:branch-and-bound:relax}). That relaxed DD serves two purposes: first, 
it is used to derive an upper bound $v^*(\rlx{B})$ which is compared to the 
current best known solution (line \ref{alg:branch-and-bound:2b}). This gives us 
a chance to prune the unexplored state space under $u$ when $v^*(\rlx{B})$ 
guarantees it does not contain any better solution than the current best. 
The second use of $\rlx{B}$ happens when $v^*(\rlx{B})$ cannot provide such a 
guarantee. In that case, the exact cutset of $\rlx{B}$ is used to enumerate 
residual sub-problems which are enqueued onto the fringe (lines 
\ref{alg:bab-cutset-start}-\ref{alg:bab-cutset-end}).

A cutset for some relaxed DD $\rlx{B}$ is a subset $\mathcal{C}$ of 
the nodes from $\rlx{B}$ such that any $r-t$ path of $\rlx{B}$ goes through at 
least one node $\in \mathcal{C}$. Also, a node $u$  is said to be exact iff all 
its incoming paths lead to the same state $\sigma(u)$. From there, an exact 
cutset of $\rlx{B}$ is simply a cutset whose nodes are all exact.
Based on this definition, it is easy to convince  
oneself that an exact cutset constitutes a frontier up to which the relaxed DD 
$\rlx{B}$ and its exact counterpart $\dd{B}$ have not diverged. And, because it 
is a 
cutset, the nodes composing that frontier cover all paths from both $\dd{B}$ and 
$\rlx{B}$; which guarantees the completeness of 
\alg{alg:branch-and-bound} \cite{bergman:16:branch-and-bound}.

Any relaxed-MDD admits at least one exact cutset -- e.g. the trivial $\set{r}$ 
case. Often though, it is not unique and different options exist as to what 
cutset to use. It was experimentally shown by \cite{bergman:16:branch-and-bound} 
that most of  the time, the Last Exact Layer (LEL) is superior to all other 
exact cutsets in practice. LEL consists of the \emph{deepest} layer of the 
relaxed-MDD having all its nodes exact.
\begin{example} 
    In Fig.-\ref{fig:mdd:exact-restricted-relaxed} (c), the first 
    inexact node $\mathcal{M}$ occurs in layer $L_2$. Hence, the LEL cutset 
    comprises all nodes (a, b, c) from the layer $L_1$. Because $\mathcal{M}$ 
    is inexact, and because it is a parent of nodes i, j and k, these three 
    nodes are considered inexact too.
\end{example}
\vspace{-2em}
\noindent
\begin{center}
\scalebox{0.8}{
\begin{minipage}[t]{.6\textwidth}
\begin{algorithm}[H]
    \begin{algorithmic}[1]
        \STATE Create node $r$ and add it to $Fringe$ 
        \label{alg:branch-and-bound:1a}
        \STATE $\lb{x} \gets \bot$
        \STATE $\lb{v} \gets -\infty$
        \WHILE{$Fringe$ is not empty}
        \STATE $u \gets Fringe.pop()$ \label{alg:branch-and-bound:1b}
        \STATE $\rst{B} \gets Restricted(u)$ \label{alg:branch-and-bound:2a}
        \IF{$v^*(\rst{B}) > \lb{v}$}\label{alg:branch-and-bound:improve-lb-start}
        \STATE $\lb{v} \gets v^*(\rst{B})$
        \STATE $\lb{x} \gets x^*(\rst{B})$
        \ENDIF\label{alg:branch-and-bound:improve-lb-end}
        \IF{$\rst{B}$ is not exact} \label{alg:branch-and-bound:prune}
        \STATE $\rlx{B} \gets Relaxed(u)$\label{alg:branch-and-bound:relax}
        \IF{$v^*(\rlx{B}) > \lb{v}$} \label{alg:branch-and-bound:2b}
        \FORALL{$u' \in \rlx{B}.exact\_cutset()$} \label{alg:bab-cutset-start}
        \label{alg:branch-and-bound:cutset}
        \STATE $Fringe.add(u')$
        \ENDFOR \label{alg:bab-cutset-end}
        \ENDIF \label{alg:branch-and-bound:3}
        \ENDIF
        \ENDWHILE
        \RETURN $\lst{\lb{x}, \lb{v}}$ \label{alg:branch-and-bound:4}
    \end{algorithmic}
    \caption{Branch-And-Bound with DD}
    \label{alg:branch-and-bound}
\end{algorithm}
\end{minipage}
}
\hfill
\scalebox{0.8}{
    \begin{minipage}[t]{.6\textwidth}
        \begin{algorithm}[H]
            \begin{algorithmic}[1]
                \color{gray}
                \STATE Create node $r$ and add it to $Fringe$ 
                \STATE $\lb{x} \gets \bot$
                \STATE $\lb{v} \gets -\infty$
                \WHILE{$Fringe$ is not empty}
                    \STATE $u \gets Fringe.pop()$
                    \color{black}
                    \IF{$v|_u^* \le \lb{v}$}        
                        \label{alg:reasoning:local-bounds:pop-skip}
                        \STATE \textbf{continue}
                    \ENDIF
                    \color{gray}
                    \STATE $\rst{B} \gets Restricted(u)$ 
                    \IF{$v^*(\rst{B}) > \lb{v}$}
                        \STATE $\lb{v} \gets v^*(\rst{B})$
                        \STATE $\lb{x} \gets x^*(\rst{B})$
                    \ENDIF
                    \IF{$\rst{B}$ is not exact}
                        \STATE $\rlx{B} \gets Relaxed(u)$
                        \IF{$v^*(\rlx{B}) > \lb{v}$}
                            \color{black}
                            \FORALL{$u' \in \rlx{B}.exact\_cutset()$}
                                \IF{$v|_{u'}^* > \lb{v}$} 
                                    \label{alg:reasoning:local-bounds:skip-enqueue}
                                    \STATE $Fringe.add(u')$
                                    \label{alg:reasoning:local-bounds:do-enqueue}
                                \ENDIF
                            \ENDFOR
                            \color{gray}
                        \ENDIF
                    \ENDIF
                \ENDWHILE
                \RETURN $\lst{\lb{x}, \lb{v}}$
            \end{algorithmic}
            \caption{Local bound pruning}
            \label{alg:reasoning:local-bounds}
        \end{algorithm}
    \end{minipage}
}
\end{center}

\section{Improving the filtering of branch-and-bound MDD}
In the forthcoming paragraphs, we introduce the local bound and present the 
rough upper bound: two reasoning techniques to reinforce the pruning strength 
of \alg{alg:branch-and-bound}.

\subsection{Local bounds (LocB)}
\label{sec:local-bounds}
Conceptually, pruning with local bounds is rather simple: a relaxed MDD 
$\rlx{B}$ provides us with \emph{one} upper bound $v^*(\rlx{B})$ 
on the optimal value of the objective function for some given sub-problem. 
However, in the event where $v^*(\rlx{B})$ is greater than the best known lower 
bound $\lb{v}$ (best current solution) nothing guarantees that all nodes from 
the exact cutset of $\rlx{B}$ admit a longest path to $t$ with a length of 
$v^*(\rlx{B})$. Actually, this is quite unlikely. This is why we propose to 
attach a \emph{``local'' upper bound} to each node of the cutset. This local 
upper bound -- denoted $v|_u^*$ for some cutset node $u$ -- simply records 
the length of the longest r-t path passing through $u$ in the relaxed MDD 
$\rlx{B}$.

In other words, LocB  allows us to refine the information provided by a relaxed 
DD $\rlx{B}$. On one hand, $\rlx{B}$ provides us with $v^*(\rlx{B})$ which is 
the length of the longest r-t path in $\rlx{B}$. As such, it provides an upper 
bound on the optimal value that can be reached from the root node of $\rlx{B}$. 
With the addition of LocB, the relaxed DD provides us with an additional piece 
of information. For each individual node $u$ in the exact cutset of $\rlx{B}$, 
it defines the value $v|_u^*$ which is an upper bound on the value attainable 
from that node.

As shown in \alg{alg:reasoning:local-bounds}, the value $v|_u^*$ can prove 
useful at two different moments. First, in the event where $v|_u^* \le \lb{v}$, 
this value can serve as a justification to not enqueue the subproblem $u$ (line 
\ref{alg:reasoning:local-bounds:skip-enqueue}) since exhausting this subproblem 
will yield no better solution than $\lb{v}$. 
%%%%%
More formally, by definition of a cutset and of LocB, it must be the case that 
the longest r-t path of $\rlx{B}$ traverses one of the cutset nodes $u$ and thus 
that $v^*(\rlx{B}) = v|_u^*$ (where $v|_u^*$ is the local bound of $u$). Hence 
we have: $\exists u \in \text{~cutset of~} \rlx{B} : v^*(\rlx{B}) = v|^*_u$.
However, because $v^*(\rlx{B})$ is the length of the \emph{longest} r-t path of 
$\rlx{B}$, there may exist cutset nodes that only belong to r-t paths shorter 
than $v^*(\rlx{B})$. That is: $
\forall u’ \in \text{~cutset of~} \rlx{B} : v^*(\rlx{B}) \ge v|_{u’}^*$.
Which is why $v|_{u’}^*$ can be stricter than $v^*(\rlx{B})$ and hence let 
LocB be stronger at pruning nodes from the frontier. 
%%%%

The second time when $v|_u^*$  might come in handy occurs when the node $u$ is 
popped out of the fringe (line \ref{alg:reasoning:local-bounds:pop-skip}). 
Indeed, because the fringe is a global priority queue, any node that has been 
pushed on the fringe can remain there for a long period of time. Thus, chances 
are that the value $\lb{v}$ has increased between the moment when the node was 
pushed onto the fringe (line \ref{alg:reasoning:local-bounds:do-enqueue}) and 
the moment when it is popped out of it. Hence, this gives us an additional 
chance to completely skip the exploration of the sub-problem rooted in $u$.

Let us illustrate that with the relaxed MDD shown on 
Fig.\ref{fig:reasoning:local-bounds}, for which the exact cutset comprises the 
highlighted nodes $a$ and $b$. Please note that because this scenario may occur 
at any time during the problem resolution, we will assume that the fringe is not 
empty when it starts.
Assuming that the current best solution $\lb{v}$ is $20$ when one explores the 
pictured subproblem, we are certain that exploring the subproblem rooted in $a$ 
is a waste of time, because the local bound $v|_a^*$ is only $16$. 
Also, because the fringe was not empty, it might be the case that $b$ was left 
on the 
fringe for a long period of time. And because of this, it might be the case that 
the best known value $\lb{v}$ was improved between the moment when $b$ was 
pushed on the fringe and the moment when it was popped out of it.
Assuming that $\lb{v}$ has improved to $110$ when 
$b$ is popped out of the fringe, it may safely be skipped because $v|_b^*$ 
guarantees that an exploration of $b$ will not yield a better solution than 
$102$.

\alg{alg:computing-local-bounds} describes the procedure to compute the local 
bound $v|^*_u$ of each node $u$ belonging to the exact cutset of a relaxed MDD 
$\rlx{B}$. Intuitively, this is achieved by doing a bottom-up traversal of 
$\rlx{B}$, starting at $t$ and stopping when the traversal crosses the last 
exact layer (line \ref{alg:computing-local-bounds:end-cond}). During that 
bottom-up traversal, the algorithm marks the nodes that are reachable from $t$.
This way, it can avoid the traversal of dead-end nodes. Also, 
\alg{alg:computing-local-bounds} maintains a value $v^*_{\uparrow_t}(u)$ for 
each node $u$ it encounters. This value represents the length of the longest u-t 
path. Afterwards (line \ref{alg:computing-local-bounds:save-lb}), it is summed 
with the length of the longest r-u path $v^*_{r-u}$ to derive the exact value of 
the local bound $v|^*_u$.

\begin{algorithm}
    \begin{algorithmic}[1]
        \STATE $lel \gets$ \text{Index of the last exact layer}
        \STATE $v^*_{\uparrow_t}(u) \gets -\infty ~\text{\textbf{for each 
        node}}~u \in \rlx{B}$ \hfill\texttt{// init. longest u-t path}
        \STATE $mark(t) \gets $true
        \STATE $v^*_{\uparrow_t}(t) \gets 0$ \hfill\texttt{// longest t-t path}
        \FORALL{$i = n$ to $lel$}\label{alg:computing-local-bounds:end-cond}
           \FORALL{node $u \in L_{i}$}
                \IF{$mark(u)$} 
                    \FORALL{arc $a = (u', u)$ incident to $u$}
                        \STATE $mark(u') \gets$ true
                        \STATE $v^*_{\uparrow_t}(u') \gets 
                        \max ( v^*_{\uparrow_t}(u'), v^*_{\uparrow_t}(u) + v(a) 
                        ) 
                        $  \hfill\texttt{// longest u'-t path}
                        \label{alg:computing-local-bounds:lp_from_bot}
                    \ENDFOR
                \ENDIF
           \ENDFOR
        \ENDFOR
        \FORALL{node $u \in \rlx{B}.exact\_cutset()$}
            \IF{$mark(u)$}
                \STATE $v|^*_u \gets v^*_{r-u} + v^*_{\uparrow_t}(u)$  
                \hfill\texttt{// longest r-u path + longest u-t path}
                \label{alg:computing-local-bounds:save-lb}
            \ELSE
                \STATE $v|^*_u \gets -\infty$ 
                \label{alg:computing-local-bounds:ditch-not-marked}
            \ENDIF
        \ENDFOR
    \end{algorithmic}
    \caption{Computing the local bounds}
    \label{alg:computing-local-bounds}
\end{algorithm}

\paragraph{Caveat for the First Exact Layer Cutset (FEL)}
\alg{alg:computing-local-bounds} is sufficient to cope with the cases where  
$\rlx{B}$ uses either a LEL cutset or an FC cutset. However, in the event where 
$\rlx{B}$ would implement another strategy -- such as ie the First Exact Layer 
(FEL), then \alg{alg:computing-local-bounds} would need to be adapted and 
perform a complete backward traversal of $\rlx{B}$. Concretely, this means that 
line \ref{alg:computing-local-bounds:end-cond} would need to read $\text{for 
    all } i \in n \text{ down to } 0$ instead of $\text{down to } lel$. This 
change is mandatory. Failing to implement it could result in too many nodes 
being pruned away because of the line 
\ref{alg:computing-local-bounds:ditch-not-marked}.

\subsubsection{Complexity matters...} 
As shown in \alg{alg:computing-local-bounds}, the computation of the local 
bound $v|^*_u$ of some node $u$ from the exact cutset of a relaxed MDD 
$\rlx{B}$ only requires the ability to remember if $u$ is reachable from $t$ 
($mark(u)$) and to sum the lengths of the longest r-u path $v^*_{r-u}$ and the 
longest u-t path $v^*_{\uparrow_t}(u)$. To that end, it is sufficient to only 
store one flag and two integers in each node. Thus, the spatial complexity of 
implementing local bounds is $\Theta(1)$ per node. 

Similarly, the backwards traversal of $\rlx{B}$ as performed in 
\alg{alg:computing-local-bounds} only visits a subset of the nodes and arcs 
that have been created during the compilation of $\rlx{B}$. The 
time complexity of \alg{alg:computing-local-bounds} is thus bounded by 
$O(|U|\times|A|)$. Actually, we even know that this bound is lesser or equal to 
that of the top-down compilation because the width-bounding procedure 
reduces the number of nodes in $\rlx{B}$. Hence it reduces 
the number of nodes potentially visited by \alg{alg:computing-local-bounds}. It 
follows that the use of local bounds has no impact on the time complexity of 
the derivation of a relaxed MDD.

\subsection{Rough upper bound (RUB)}\label{sec:rub}
Rough upper bound pruning departs from the following observation: assuming the 
knowledge of a lower bound $\underline{v}$ on the value of $v^*$, and assuming 
that one is able to swiftly compute a rough upper bound $\overline{v_s}$ on the 
optimal value $v^*_s$ of the subproblem rooted in state $s$; any node $u$ of a 
MDD having a rough upper bound
$\overline{v_{\sigma(u)}} \le \underline{v}$ may be discarded as it is 
guaranteed not to improve the best known solution. This is pretty much the same 
reasoning that underlies the whole branch-and-bound idea. But here, it is used 
to prune portions of the search space explored \emph{while compiling} 
approximate MDDs.

To implement RUB, it suffices to adapt the MDD compilation procedure (top-down , 
iterative refinement, ...) and introduce a check that avoids creating a node 
$u'$ with state $next$ when $\overline{v_{next}} \le \underline{v}$.

%%%%%%
The key to RUB effectiveness is that RUB is used while compiling the restricted 
and relaxed DDs. As such, its computation does not directly appear in Alg. 1, 
but rather is accounted within the compilations of $Restricted(u)$ and 
$Relaxed(u)$
from Alg. 1. Thus, it really is not used as yet-an-other-bound competing with 
that of line 12, but instead to speed up the computation of restricted and 
relaxed DDs. More precisely, this speedup occurs because the compilation of the 
DDs discards some nodes that would otherwise be added to the next layer of the 
DD and then further expanded, which are ruled out by RUB.
A second benefit of using RUBs is that it helps tightening the bound derived 
from a relaxed DD (Alg.1 line 12). Because the layers that are generated in a 
relaxed DD are narrower when applying RUB, there are fewer nodes exceeding the 
maximum layer width. The operator $\oplus$ hence needs to merge a smaller set of 
nodes in order to produce the relaxation. 
%%%%%%

The dynamics of RUB is graphically illustrated by 
Fig.-\ref{fig:reasoning:rough-ub} where the set of highlighted nodes can be 
safely elided since the (rough) upper bound computed in node $s$ is lesser than 
the best lower bound.

\begin{figure}[ht]
    \begin{minipage}{.48\textwidth}
        \centering
        \scriptsize
        \ctikzfig{graphs/local_bounds}
        \caption{
            An example relaxed-MDD having an exact cutset $\set{a, b}$ with 
            local bounds $v|_a^*$ and $v|_b^*$. The nodes with a simple border 
            represent exact nodes and those with a double border represent 
            ``inexact'' nodes. The edges along the longest path are displayed in 
            bold.}
        \label{fig:reasoning:local-bounds}
    \end{minipage}
    \hfill
    \begin{minipage}{.48\textwidth}
        \centering
        \scriptsize
        \ctikzfig{graphs/rough_ub}
        \caption{
            Assuming a lower bound $\lb{v}$ of 100 and a rough upper bound 
            $\ub{v_s}$ of 42 for the node $s$, all the highlighted nodes (in 
            red, with a dashed border) may be pruned from the MDD.}
        \label{fig:reasoning:rough-ub}
    \end{minipage}
\end{figure}

\subsubsection{Important Note}
It is important to understand that because the RUB is computed at each node of 
each restricted and relaxed MDD compiled during the instance resolution, it must 
be extremely inexpensive to compute. This is why RUB is best obtained from a 
fast and simple problem specific procedure.

\section{Experimental Study}
In order to evaluate the impact of the pruning techniques 
proposed above, we conducted a series of experiments on four problems.
In particular, we conducted experiments on the Maximum Independent Set Problem 
(MISP), the Maximum Cut Problem (MCP), the Maximum Weighted 2-Satisfiablility 
Problem (MAX2SAT) and the Traveling Salesman Problem with Time Windows (TSPTW). 
For the first three problems, we generated sets of random instances which we 
attempted to solve with different configurations of our own open source solver 
written in Rust 
\cite{gillard:20:ddo}\footnote{https://github.com/xgillard/ddo}. For TSPTW, we 
reused openly available sets of benchmarks which are usually used to assess the 
efficiency of new solvers for TSPTW\cite{lopez:20:tsptw}. 
Thanks to the generic nature of our framework, the model and all heuristics used 
to solve the instances were the same for all experiments. This allowed us to 
isolate the impact of RUB and LocB on the solving performance and neutralize 
unrelated factors such as variable ordering. Indeed, the only variations between 
the different solver flavors relate to the presence (or absence) of RUB and LocB.
All experiments were run on the same physical machine equipped with an AMD6176 
processor and 48GB of RAM. A maximum time limit of 1800 seconds was allotted to 
each configuration to solve each instance.

The models we used for MISP, MCP and MAX2SAT are direct translations of 
the ones described in \cite{bergman:16:branch-and-bound}. The details of the 
RUBs we formulated for these problems are given in the appendixes 
\ref{appendix:rub:misp}, \ref{appendix:rub:mcp} and \ref{appendix:rub:max2sat}. 
The DP formulation we used for TSPTW, its relaxation and RUB are detailed in 
appendix \ref{appendix:rub:tsptw}.

\paragraph{MISP.}\label{sec:xp:misp}
To assess the impact of RUB and LocB on MISP, we generated random graphs based 
on the Erdos-Renyi model G(n, p) \cite{erdos-renyi:59} with the number of 
vertices n = 250, 500, 750, 1000, 1250, 1500, 1750 and the probability of having 
an edge connecting any two vertices p = 0.1, 0.2, ... , 0.9. The weight of the 
edges in the generated graphs were drawn uniformly from the set $\{-5, -4, -3, 
-2, -1, 1, 2, 3, 4, 5\}$. We generated 10 instances for each combination of size 
and density (n, p).

\paragraph{MCP.}\label{sec:xp:mcp}
In line with the strategy used for MISP, we generated random MCP instances as 
random graphs based on the Erdos-Renyi model G(n, p). These graphs were generated 
with the number of vertices n = 30, 40, 50 and the probability p 
of connecting any two vertices = 0.1, 0.2, 0.3, .., 0.9. The 
weights of the edges in the generated graphs were drawn uniformly among $\{-1, 
1\}$. Again, we generated 10 instances per combination n, p. 

\paragraph{MAX2SAT.}\label{sec:xp:max2sat}
Similar to the above, we used random graphs based the Erdos-Renyi model G(n, p) 
to derive MAX2SAT instances. To this end, we produced graphs with n = 
60, 80, 100, 200, 400, 1000 (hence instances with 30, 40, 50, 100, 200 and 500 
variables) and p = 0.1, 0.2, 0.3, .. , 0.9. For each combination of size 
(n) and density (p), we generated 10 instances. The weights of the clauses
in the generated instances were drawn uniformly from the set $\{1, 2, 3, 5, 6, 
7, 8, 9, 10\}$.

\paragraph{TSPTW.}\label{sec:xp:tsptw}
To evaluate the effectiveness of our rules on TSPTW, we used the 467 instances 
from the following suites of benchmarks, which are usually used to assess the 
efficiency of new TSPTW solvers. AFG \cite{bench:afg}, Dumas \cite{bench:dumas}, 
Gendreau-Dumas \cite{bench:gendreau-dumas}, Langevin \cite{bench:langevin}, 
Ohlmann-Thomas \cite{bench:ohlmann-thomas}, 
Solomon-Pesant \cite{bench:solomon-pesant} and 
Solomon-Potvin-Bengio \cite{bench:potvin-bengio}.

\begin{figure}
    \centering
    \includegraphics[width=\textwidth]{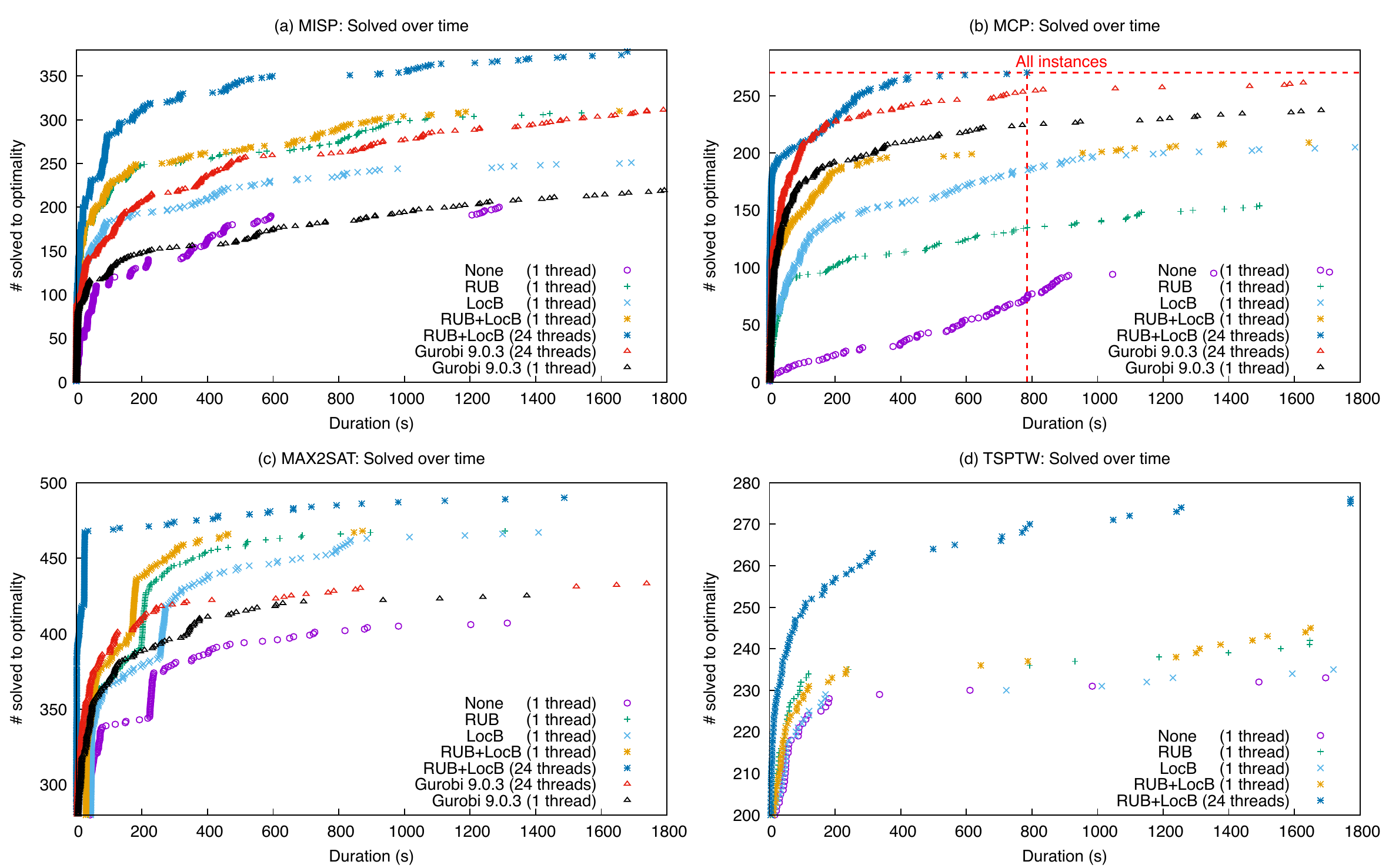}
    \caption{Number of solved instances over time for each considered problem}
    \label{fig:xp:solved}
\end{figure}

Figure \ref{fig:xp:solved} gives an overview of the results from our 
experimental study. It respectively depicts the evolution over time 
of the number of instances solved by each technique for MISP (a), MCP (b) and 
MAX2SAT (c) and TSPTW (d). 

As a first step, our observation of the graphs will focus on the differences 
that arise between the single threaded configurations of our \emph{ddo} solvers. 
Then, in a second phase, we will incorporate an existing state-of-the-art ILP 
solver (Gurobi 9.0.3) in the comparison. Also, because both Gurobi and our 
\emph{ddo} library come with built-in parallel computation capabilities, we will 
consider both the single threaded and parallel (24 threads) cases. This second 
phase, however, only bears on MISP, MCP and MAX2SAT by lack of a Gurobi TSPTW 
model.

\paragraph{DDO configurations}
The first observation to be made about the four graphs in 
Fig.\ref{fig:xp:solved}, is that for all considered problems, both RUB and 
LocB outperformed the 'do-nothing' strategy; thereby showing the relevance of 
the rules we propose. It is not clear however which of the two rules brings the 
most improvement to the problem resolution. Indeed, RUB seems to be the driving 
improvement factor for MISP (a) and TSPTW (d) and the impact of LocB appears to 
be moderate or weak on these problems. However, it has a much higher impact  
for MCP (b) and MAX2SAT (c). In particular, LocB appears to be the driving 
improvement factor for MCP (b). This is quite remarkable given that LocB 
operates in a purely black box fashion, without any problem-specific knowledge. 
Finally, it should also be noted that the use of RUB and LocB are not mutually 
exclusive. Moreover, it turns out that for all considered problems, the 
combination RUB+LocB improved the situation over the use of any single rule. 
%Indeed, this combination yields the virtual best solver for all tested 
%problems. 

Furthermore, Fig.\ref{fig:xp:gap} confirms the benefit of
using both RUB \emph{and} LocB together rather than using any single technique. 
For each problem, it measures the “performance” of using RUB+LocB vs the best 
single technique through the end gap. The end gap is defined as 
$\left(100 * \frac{|UB|-|LB|}{|UB|}\right)$. This metric allows us to account 
for all instances, including the ones that could not be solved to optimality. 
Basically, a small end gap means that the solver was able to confirm a tight 
confidence interval of the optimum. Hence, a smaller gap is better. On each 
subgraphs of Fig.\ref{fig:xp:gap}, the distance along the x-axis represents the 
end gap for reach instance when using both RUB and LocB whereas the distance 
along y-axis represents the end gap when using the best single technique for the 
problem at hand. Any mark above the diagonal shows an instance for which using 
both RUB and LocB helped reduce the end gap and any mark below that line 
indicates an instance where it was detrimental.

From graphs \ref{fig:xp:gap}-a, \ref{fig:xp:gap}-c and \ref{fig:xp:gap}-d it 
appears that the combination RUB+LocB supersedes the use of RUB only. 
Indeed the vast majority of the marks sit above the diagonal and the rest on it. 
This indicates a beneficial impact of using both techniques even for the hardest 
(unsolved) instances. The case of MCP (graph \ref{fig:xp:gap}-b) is less clear 
as most of the marks sit on the diagonal. Still, we can only observe three marks 
below the diagonal and a bit more above it. Which means that even though the use 
of RUB in addition to LocB is of little help in the case of MCP, its use does 
not degrade the performance for that considered problem.

\begin{figure}
    \centering
    \includegraphics[scale=.55]{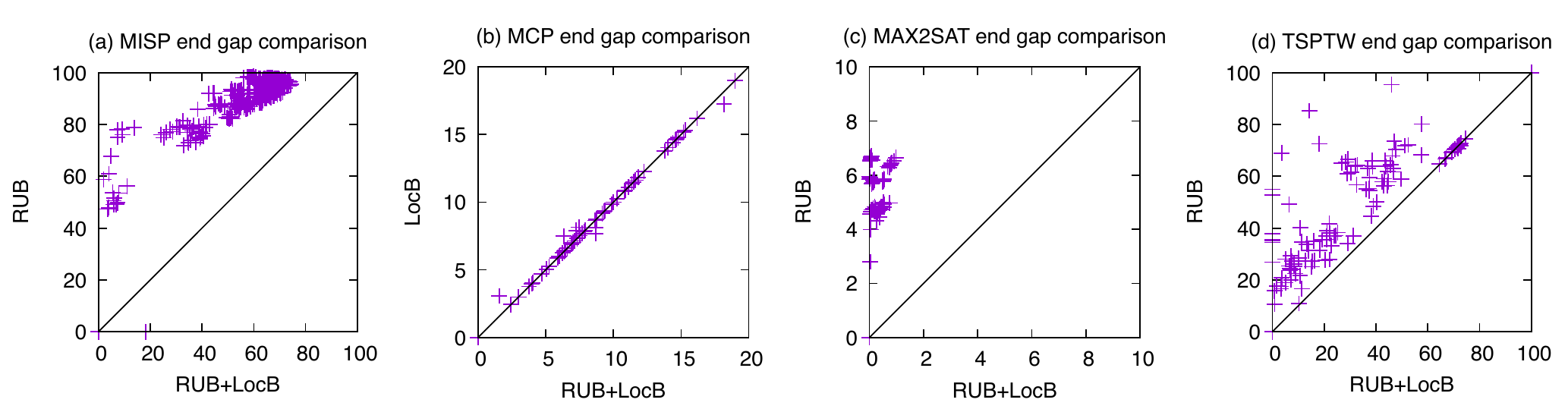}
    \caption{End gap: The benefit of using both techniques vs the best single 
    one}
    \label{fig:xp:gap}
\end{figure}

\paragraph{Comparison with Gurobi 9.0.3}
The first observation to be made when comparing the performance of Gurobi vs the 
DDO configurations, is that when running on a single thread, ILP outperforms the 
basic DDO approach (without RUB and LocB). Furthermore, Gurobi turns out to be 
the best single threaded solver for MCP by a fair margin. However, in the MISP 
and MAX2SAT cases, Fig.~\ref{fig:xp:solved} shows that the DDO solvers 
benefitting 
from RUB and LocB were able to solve more instances and to solve them faster 
than Gurobi. Which underlines the importance of RUB and LocB.

When lifting the one thread limit, one can see that the DD-based approach 
outperform ILP on each of the considered problems. In particular, in the case of 
MCP for which Gurobi is the best single threaded option; our DDO solver was 
able to find and prove the optimality of all tested instances in a little less 
than 800 seconds. The ILP solver, on the other end, was not able to prove the 
optimality of the 9 hardest instances within 30 minutes. Additionally, we also 
observe that in spite of the performance gains of MIP when running in parallel, 
Gurobi fails to solve as many MISP and MAX2SAT instances and to solve them as 
fast as the single threaded DDO solvers with RUB and LocB. This emphasizes once 
more the relevance of our techniques. 
It also shows that the observation from \cite{bergman:14:parallel} still hold 
today: despite the many advances of MIP the DDO approach still scales better 
than MIP on the considered problems when invoked in parallel.

\section{Previous work}\label{sec:related}
DDO emerged in the mid' 2000's when \cite{hooker:06:discrete-optimization} 
proposed to use decision diagrams as a way to solve discrete optimization 
problems to optimality. More or less concomitantly, \cite{hooker:07:relaxed-mdd} 
devised relaxed-MDD even though the authors envisioned its use as a CP 
constraint store rather than a means to derive tight upper bounds for 
optimization problems. Then, the relationship between decision diagrams and 
dynamic programming was clarified by \cite{hooker:13:dd-dp}. 

Recently, Bergman, Ciré and van Hoeve investigated the various ways to 
compile decision diagrams for optimization (top-down, construction by 
separation) \cite{cire:14:decision-diagrams}. They also investigated the 
heuristics used to parameterize these DD compilations. In particular, they 
analyzed the impact of variable ordering in
\cite{cire:14:decision-diagrams,bergman:14:optimization-bounds} and node 
selection heuristics (for merge and deletion) in 
\cite{bergman:14:optimization-bounds}. Doing so, they empirically demonstrated 
the crucial impact of variable ordering on the tightness of the derived bounds 
and highlighted the efficiency of minLP as a node selection heuristic. Later 
on, the same authors proposed a complete branch-and-bound algorithm based on DDs
\cite{bergman:16:branch-and-bound}. This is the algorithm which we propose to 
adapt with extra reasoning mechanisms and for which we provide a generic 
open-source implementation in Rust \cite{gillard:20:ddo}. The impressive 
performance of DDO triggered some theoretical research to analyze the quality of 
approximate MDDs \cite{bergman:16:theoretical} and the correctness of the 
relaxation operators \cite{hooker:17:job-sequencing}.

This gave rise to new lines of work. The first one focuses on the 
resolution of a larger class of optimization problems; chief of which 
multi-objective problems \cite{bergman:16:multiobjective} and
problems with a non-linear objective function. These are either solved by 
decomposition \cite{bergman:16:multiobjective} or by using DDO to strengthen 
other IP techniques \cite{davarnia:18:approx-non-linear}. A second trend aims 
at hybridizing DDO with other IP techniques. For instance, by using Lagrangian 
relaxation \cite{hooker:19:improved-job-sequencing} or by solving a MIP 
\cite{bergman:17:optimal-bdd-relax} to derive with very tight bounds. But the 
other direction is also under active investigation: for example, 
\cite{tjandraatmadja:18:decision-diagrams,tjandraatmadja:19:cuts-from-relaxed} 
use DD to derive tight bounds which are used to replace LP relaxation in a 
cutting planes solver. Very recently, a third hybridization approach has been 
proposed by Gonzàlez et al.\cite{gonzalez:20:mip-dp-ml}. It adopts the 
branch-and-bound MDD perspective, but whenever an upper bound is to be derived, 
it uses a trained classifier to decide whether the upper bound is to be computed 
with ILP or by developing a fixed-width relaxed MDD. 

The techniques (ILP-cutoff pruning and ILP-cutoff heuristic) proposed by 
Gonzalez et al.\cite{gonzalez:20:mip-dp-ml} are related to RUB and LocB in the 
sense that all techniques aim at reducing the search space of the problem. 
However, they fundamentally differ as ILP-cutoff pruning acts as a  replacement 
for the compilation of a relaxed MDD whereas the goal of RUB is to speed up the 
development of that relaxed MDD by removing nodes \emph{while} the MDD is being 
generated. The difference is even bigger in the case of ILP-cutoff heuristic vs 
LocB: the former is used as a primal heuristic while LocB is used to filter out 
sub-problems that can bear no better solution. In that sense, LocB belongs more 
to the line of work started by 
\cite{hooker:07:relaxed-mdd,hooker:08:propagting-separable-equalities,hooker:10:mdd-cp}:
it enforces the constraint $lb \le f(x) \le ub$ and therefore provokes the 
deletion of nodes and arcs that cannot lead to the optimal solution.

More recently, Horn et al explored an idea in \cite{horn:2021} which closely 
relates to RUB. They use ``fast-to-compute dual bounds'' as an admissible 
heuristic to guide the compilation of MDDs in an A* fashion for the 
prize-collecting TSP. It prunes portions of the state space during the MDD 
construction, similarly to when RUB is applied. Our approach differs from that 
of \cite{horn:2021} in that we attempt to incorporate problem specific knowledge 
in a framework that is otherwise fully generic. More precisely, it is perceived 
here as a problem-specific pruning that exploits the combinatorial structure 
implied by the state variables. It is independent of other MDD compilation 
techniques, e.g., our techniques are compatible with node merge ($\oplus$) 
operators and other methodologies defined in the DDO literature. We also 
emphasize that, as opposed to more complex LP-based heuristics that are now 
typical in A* search, we investigate quick methodologies that are also easy to 
incorporate in a MDD branch and bound.

\section*{Conclusion and future work}\label{sec:conclusion}
This paper presented and evaluated the impact of the local bound and rough upper 
bound techniques to strengthen the pruning of the branch-and-bound MDD 
algorithm. Our experimental study on MISP, MCP, MAX2SAT and TSPTW confirmed the 
relevance of these techniques. In particular, our experiments have shown that 
devising a fast and simple rough upper bound is worth the effort as it can 
significantly boost the efficiency of a solver. Similarly, our experiments 
showed that the use of local bound can significantly improve the efficiency of 
DDO solver despite its problem agnosticism. Furthermore, it revealed that a 
combination of RUB and LocB supersedes the benefit of any single reasoning 
technique. 
These results are very promising and we believe that the public 
availability of an open source DDO framework implementing RUB and LocB might 
serve as a basis for novel DP formulation for classic problems. 

\newpage
\bibliographystyle{splncs04}
\bibliography{ddo}

\begin{thebibliography}{10}
\providecommand{\url}[1]{\texttt{#1}}
\providecommand{\urlprefix}{URL }
\providecommand{\doi}[1]{https://doi.org/#1}

\bibitem{hooker:07:relaxed-mdd}
Andersen, H.R., Hadzic, T., Hooker, J.N., Tiedemann, P.: A constraint store
  based on multivalued decision diagrams. In: Bessière, C. (ed.) Principles
  and Practice of Constraint Programming. LNCS, vol.~4741, pp. 118--132.
  Springer (2007)

\bibitem{bench:afg}
Ascheuer, N.: Hamiltonian path problems in the on-line optimization of flexible
  manufacturing systems  (1996)

\bibitem{bellman:54}
Bellman, R.: The theory of dynamic programming. Bulletin of the American
  Mathematical Society  \textbf{60}(6),  503--515 (11 1954),
  \url{https://projecteuclid.org:443/euclid.bams/1183519147}

\bibitem{bergman:16:multiobjective}
Bergman, D., Cire, A.A.: Multiobjective optimization by decision diagrams. In:
  Rueher, M. (ed.) Principles and Practice of Constraint Programming. LNCS,
  vol.~9892, pp. 86--95. Springer (2016)

\bibitem{bergman:16:theoretical}
Bergman, D., Cire, A.A.: Theoretical insights and algorithmic tools for
  decision diagram-based optimization. Constraints  \textbf{21}(4),  533–556
  (2016). \doi{10.1007/s10601-016-9239-9},
  \url{https://doi.org/10.1007/s10601-016-9239-9}

\bibitem{bergman:17:optimal-bdd-relax}
Bergman, D., Cire, A.A.: On finding the optimal bdd relaxation. In: Salvagnin,
  D., Lombardi, M. (eds.) Integration of AI and OR Techniques in Constraint
  Programming. LNCS, vol. 10335, pp. 41--50. Springer (2017)

\bibitem{bergman:14:optimization-bounds}
Bergman, D., Cire, A.A., van Hoeve, W.J., Hooker, J.N.: Optimization bounds
  from binary decision diagrams. INFORMS Journal on Computing  \textbf{26}(2),
  253--268 (2014). \doi{10.1287/ijoc.2013.0561},
  \url{https://doi.org/10.1287/ijoc.2013.0561}

\bibitem{bergman:16:branch-and-bound}
Bergman, D., Cire, A.A., van Hoeve, W.J., Hooker, J.N.: Discrete optimization
  with decision diagrams. INFORMS Journal on Computing  \textbf{28}(1),  47--66
  (2016). \doi{10.1287/ijoc.2015.0648},
  \url{https://doi.org/10.1287/ijoc.2015.0648}

\bibitem{bergman:14:parallel}
Bergman, D., Cire, A.A., Sabharwal, A., Samulowitz, H., Saraswat, V., van
  Hoeve, W.J.: Parallel combinatorial optimization with decision diagrams.
  International Conference on AI and OR Techniques in Constriant Programming
  for Combinatorial Optimization Problems pp. 351--367 (2014)

\bibitem{mcmillan:92:symbolic-model-checking}
Burch, J., E.M., C., K.L., M., D.L., D., H.L., H.: Symbolic model checking:
  {$10^{20}$} states and beyond. Information and Computation  \textbf{98}(2),
  142--170 (1992). \doi{10.1016/0890-5401(92)90017-A},
  \url{https://doi.org/10.1016/0890-5401(92)90017-A}

\bibitem{cire:14:decision-diagrams}
Cire, A.A.: Decision Diagrams for Optimization. Ph.D. thesis, Carnegie Mellon
  University Tepper School of Business (2014)

\bibitem{cormen:09:algorithms}
Cormen, T.H., Leiserson, C.E., Rivest, R.L., Stein, C.: Introduction to
  algorithms. MIT press (2009)

\bibitem{davarnia:18:approx-non-linear}
Davarnia, D., van Hoeve, W.J.: Outer approximation for integer nonlinear
  programs via decision diagrams (2018)

\bibitem{bench:dumas}
Dumas, Y., Desrosiers, J., Gelinas, E., Solomon, M.M.: An optimal algorithm for
  the traveling salesman problem with time windows. Operations research
  \textbf{43}(2),  367--371 (1995)

\bibitem{erdos-renyi:59}
Erd\"{o}s, P., R\'{e}nyi, A.: On random graphs i. Publicationes Mathematicae
  Debrecen  \textbf{6}, ~290 (1959)

\bibitem{bench:gendreau-dumas}
Gendreau, M., Hertz, A., Laporte, G., Stan, M.: A generalized insertion
  heuristic for the traveling salesman problem with time windows. Operations
  Research  \textbf{46}(3),  330--335 (1998)

\bibitem{gillard:20:ddo}
Gillard, X., Schaus, P., Copp\'{e}, V.: Ddo, a generic and efficient framework
  for mdd-based optimization. Accepted at the International Joint Conference on
  Artificial Intelligence (IJCAI-20); DEMO track (2020)

\bibitem{gonzalez:20:mip-dp-ml}
Gonzalez, J.E., Cire, A.A., Lodi, A., Rousseau, L.M.: Integrated integer
  programming and decision diagram search tree with an application to the
  maximum independent set problem. Constraints pp. 1--24 (2020)

\bibitem{hooker:08:propagting-separable-equalities}
Had{\v{z}}i{\'c}, T., Hooker, J., Tiedemann, P.: Propagating separable
  equalities in an mdd store. In: CPAIOR. pp. 318--322 (2008)

\bibitem{hooker:10:mdd-cp}
Hoda, S., Van~Hoeve, W.J., Hooker, J.N.: A systematic approach to mdd-based
  constraint programming. In: International Conference on Principles and
  Practice of Constraint Programming. pp. 266--280. Springer (2010)

\bibitem{hooker:13:dd-dp}
Hooker, J.N.: Decision diagrams and dynamic programming. In: Gomes, C.,
  Sellmann, M. (eds.) Integration of AI and OR Techniques in Constraint
  Programming. LNCS, vol.~7874, pp. 94--110. Springer (2013)

\bibitem{hooker:17:job-sequencing}
Hooker, J.N.: Job sequencing bounds from decision diagrams. In: Beck, J.C.
  (ed.) Principles and Practice of Constraint Programming. LNCS, vol. 10416,
  pp. 565--578. Springer (2017)

\bibitem{hooker:19:improved-job-sequencing}
Hooker, J.N.: Improved job sequencing bounds from decision diagrams. In:
  Schiex, T., de~Givry, S. (eds.) Principles and Practice of Constraint
  Programming. LNCS, vol. 11802, pp. 268--283. Springer (2019)

\bibitem{hooker:06:discrete-optimization}
Hooker, J.: Discrete global optimization with binary decision diagrams. GICOLAG
  2006  (2006)

\bibitem{horn:2021}
Horn, M., \~Maschler, J., \~Raidl, G.R., \~R\"onnberg, E.: A*-based
  construction of decision diagrams for a prize-collecting scheduling problem.
  Computers \& Operations Research  \textbf{126},  105125 (2021).
  \doi{https://doi.org/10.1016/j.cor.2020.105125},
  \url{http://www.sciencedirect.com/science/article/pii/S0305054820302422}

\bibitem{bench:langevin}
Langevin, A., Desrochers, M., Desrosiers, J., G{\'e}linas, S., Soumis, F.: A
  two-commodity flow formulation for the traveling salesman and the makespan
  problems with time windows. Networks  \textbf{23}(7),  631--640 (1993)

\bibitem{lopez:20:tsptw}
López-Ibáñez, M., Blum, C.: Benchmark instances for the travelling salesman
  problem with time windows. Online (2020),
  \url{http://lopez-ibanez.eu/tsptw-instances}

\bibitem{bench:ohlmann-thomas}
Ohlmann, J.W., Thomas, B.W.: A compressed-annealing heuristic for the traveling
  salesman problem with time windows. INFORMS Journal on Computing
  \textbf{19}(1),  80--90 (2007)

\bibitem{regin:15:mdd}
Perez, G., R\'egin, J.C.: Efficient operations on mdds for building constraint
  programming models. In: Proceedings of the Twenty-Fourth International Joint
  Conference on Artificial Intelligence (IJCAI-15). pp. 374--380 (2015)

\bibitem{bench:solomon-pesant}
Pesant, G., Gendreau, M., Potvin, J.Y., Rousseau, J.M.: An exact constraint
  logic programming algorithm for the traveling salesman problem with time
  windows. Transportation Science  \textbf{32}(1),  12--29 (1998)

\bibitem{bench:potvin-bengio}
Potvin, J.Y., Bengio, S.: The vehicle routing problem with time windows part
  ii: genetic search. INFORMS journal on Computing  \textbf{8}(2),  165--172
  (1996)

\bibitem{tjandraatmadja:18:decision-diagrams}
Tjandraatmadja, C.: Decision Diagram Relaxations for Integer Programming. Ph.D.
  thesis, Carnegie Mellon University Tepper School of Business (2018)

\bibitem{tjandraatmadja:19:cuts-from-relaxed}
Tjandraatmadja, C., van Hoeve, W.J.: Target cuts from relaxed decision
  diagrams. INFORMS Journal on Computing  \textbf{31}(2),  285--301 (2019).
  \doi{10.1287/ijoc.2018.0830}, \url{https://doi.org/10.1287/ijoc.2018.0830}

\bibitem{helene:18:compact-mdd}
Verhaeghe, H., Lecoutre, C., Schaus, P.: Compact-mdd: Efficiently filtering (s)
  mdd constraints with reversible sparse bit-sets. In: Proceedings of the
  Twenty-Seventh International Joint Conference on Artificial Intelligence
  (IJCAI-18). pp. 1383--1389 (2018)

\end{thebibliography}

\appendix
\section{MISP}\label{appendix:rub:misp}
Assuming, the MISP DP formulation from \cite{bergman:16:branch-and-bound}, where 
each state $s$ from the DP model is a bit vector $s_0..s_n$. And each bit $s_i$ 
from $s$ indicates whether or not the vertex $i$ can possibly be part of the 
maximum independent set \emph{given the decisions labelling the path between the 
root state $r$ and $s$}. The merge operator $\oplus$ simply yields the union of 
the states that must be combined. The $\Gamma$ operator used to relax edges 
entering merged node does not alter the weight of these edges.

A RUB for MISP can be computed as 
$$
\overline{v_s} = v^*(s) + \sum_{i \in 0..n-1} s_i (w_i)^+
$$

where $v^*(s)$ is the length of the longest $r-s$ path and $w_i$ denotes the 
weight of vertex $i$ and $(w_i)^+ = \max(0, w_i)$. This upper bound is correct 
because the set of vertices $\left\{i \mid s_i = 1\right\}$ is a 
superset of the optimal solution to the subproblem rooted in $s$. (Proof is 
trivial).

\section{MCP}\label{appendix:rub:mcp}
\subsection{DP model}
The DP model we use for the MCP and its relaxation are those from 
\cite{bergman:16:branch-and-bound}. Their correctness is proved in the same 
paper. In this model, a state is a tuple of integer 
values representing the marginal benefit from assigning a vertex to the 
partition $T$ of the cut.

Formally it is defined as follows: 
\begin{itemize}
\item State spaces: $S_k = \left\{ s^k \in \R^n \mid s^{k}_j = 0, j=0, \dots, k 
\right\}$ with root and terminal states equal to $(0, \dots, 0)$

\item Root value: $v_r = \sum_{0 \le i < j < n } (w_{ij})^-$

\item State transition: 
$t(s^{k-1}, x_{k-1}) = (0, \dots, 0, s^{k}_{k}, \dots, s^k_{n-1})$ where 
$$
s^k_l = 
\left\{
\begin{array}{ll}
        s^{k-1}_l + w_{kl} &\text{~if~} x_{k-1} = S \\
        s^{k-1}_l - w_{kl} &\text{~if~} x_{k-1} = T \\
\end{array}
\right\}, l = k, \dots, n-1
$$

\item Transition cost: $h(s^0, x_0) = 0$ for $x_0 \in \left\{S, T\right\}$, and 
$$
h(s^k, x_k) = 
\left\{
\begin{array}{lll}
(-s^k_k)^+ &+ \sum_{\substack{l > k\\s^k_l w_{kl}} \le 0} \min\{|s^k_l|, 
|w_{kl}| 
\}, 
&\text{~if~} x_k = S \\
(s^k_k)^+ &+ \sum_{\substack{l > k\\s^k_l w_{kl}} \le 0} \min\{|s^k_l|, |w_{kl}| 
\}, 
&\text{~if~} x_k = T \\
\end{array}
\right\}
$$
\end{itemize}

\subsection{Relaxation}
For the subset $M$ of nodes included in the layer $L_k$ of the MDD 
being compiled, the $\oplus$ and $\Gamma$ operators are defined as follows.
$$
\oplus(M)_l = \left\{
\begin{array}{ll}
\substack{\min\left\{u_l\right\}\\u \in M}, &\text{~if~} u_l \ge 0 \text{~for 
all~} u \in M\\
\substack{-\min\left\{u_l\right\}\\u \in M}, &\text{~if~} u_l \le 0 \text{~for 
all~} u \in M\\
0, &\text{otherwise}
\end{array}
\right\}, l = k \dots n-1
$$

$$
\Gamma_M(v, u) = v + \sum_{l \ge k} \left(|u_l| - |\oplus(M)_l| \right)
\text{~, for all~} u \in M
$$

\subsection{RUB}
A RUB for MCP is efficiently computed as follows:
$$
\overline{s^k} = v^*(s^k) + \sum_{k \le i < n} |s^k_i| + \mathcal{V}_k + 
\mathcal{N}_k - v_r
$$

where $s^k$ is a state from the $k^\text{th}$ layer, $v^*(s^k)$ is the length of 
the longest $r-s^k$ path and $v_r$ is the root value of the DP model. 
$\mathcal{V}_k = \sum_{k \le i < j < n} (w_{ij})^+$ is an over estimation of the 
maximum cut value on the remaining induced graph (sum of the weight of all the 
positive arcs in the induced graph). $\mathcal{N}_k = \sum_{0 \le i < j < k} 
(w_{ij})^-$ is the partial sum of the weights of all negative arcs 
interconnecting vertices which have already been assigned to a partition.

This upper bound is very fast to compute (amortized $O(n)$) since $v^*(s^k)$ is 
known when node $s^k$ must be expanded and the quantities $v_r$, $\mathcal{V}_k$ 
and $\mathcal{N}_k$ can be pre-computed once for all (assuming a fixed variable 
ordering).

\begin{theorem}\label{thm:mcp:rub}
    $$
    v|^*_{s^k} \le \overline{s^k}
    $$
\end{theorem}
%
%\begin{lemma}\label{thm:mcp:rub:at-end}
%    Theorem-\ref{thm:mcp:rub} holds when all vertices are assigned to a 
%    partition ($k = n-1$).
%\end{lemma}
%
%\begin{proof} 
%    When $k=n-1$, we have $v^*(s^{k}) = v|^*_{s^{k}}$, 
%    $\sum_{k \le i < n} |s^{k}_i| =  0$, $\mathcal{V}_{k} = 0$ and 
%    $\mathcal{N}_{k} = v_r$. Hence we have $v|^*_{s^{k}} \le v|^*_{s^k} + 0 
%    + 0 + v_r - v_r$ which satisfies Lemma-\ref{thm:mcp:rub:at-end}.
%\end{proof}
%
%\begin{lemma}\label{thm:mcp:rub:at-root}
%    Theorem~\ref{thm:mcp:rub} holds at root ($k = 0$).
%\end{lemma}
%
%\begin{proof} 
%    When $k=0$, we have $s^0 = r$, $v^*(s^0) = v_r$ and $v|^*_r = v^*$. It is 
%    also true that $\sum_{0 \le i < n} |s^0_i| =  0$, $\mathcal{N}_0 = 0$ and  
%    $\mathcal{V}_0$ is the sum of the weights of all arcs having a positive 
%    weight in the graph. By definition of the maximum cut, we have $v^* \le 
%    \mathcal{V}_0$ which trivially satisfies Lemma~\ref{thm:mcp:rub:at-root}.
%\end{proof}
%
\begin{proof}
    Because the objective function to maximize is additively separable, we can 
    decompose global problem as a sum of three parts:
    \begin{enumerate}
    \item the longest $r-s^k$ path in the subproblem $\mathcal{P}_{0,k}$ bearing
        only on variables for which a decision has been made $\left[0 \dots k-1 
        \right]$. The length of this path is denoted 
        $v^*_{\mathcal{P}_{0,k}}(s^k)$.
    \item the value of the best solution to the subproblem $\mathcal{P}_{k,n}$
        bearing only on variables for which no decision has been made $\left[k 
        \dots n-1 \right]$. We use $v^*(\mathcal{P}_{k,n})$ to denote that value.
    \item the value $\mathcal{C}_k$ of the maximum cut considering only vertices
        for which a decision has been made $\left[0 \dots k-1 \right]$ and 
        vertices for which no decision has been made $\left[k \dots n-1 \right]$.
    \end{enumerate}

    Hence, we have $v|^*_{s^k} = v^*_{\mathcal{P}_{0,k}}(s^k) +  
                                 v^*(\mathcal{P}_{k,n}) +
                                 \mathcal{C}_k$ for any $0 \le k < n$.
                                 
    According to our decomposition methodology, we would like to start by 
    showing that 

    \begin{equation}\label{eq:mcp:nk}
    \begin{array}{rl}
    v^*_{\mathcal{P}_{0,k}}(s^k) &= v^*(s^k) - v_r + \mathcal{N}_k \\
    v_r(\mathcal{P}_{0,k}) + \sum_{0 \le i < k} h(s^i, x_i) &=  v^*(s^k) - v_r + 
    \mathcal{N}_k \\
    \sum_{0 \le i < j < k} (w_{ij})^- + \sum_{0 \le i < k} h(s^i, x_i) &=  
    v^*(s^k) - v_r + \mathcal{N}_k \\
    \mathcal{N}_k + \sum_{0 \le i < k} h(s^i, x_i) &=  
    v^*(s^k) - v_r + \mathcal{N}_k \\
    \mathcal{N}_k + \sum_{0 \le i < k} h(s^i, x_i) &=  
    v_r + \sum_{0 \le i < k} h(s^i, x_i) - v_r + \mathcal{N}_k \\
    \mathcal{N}_k + \sum_{0 \le i < k} h(s^i, x_i) &=  
    \sum_{0 \le i < k} h(s^i, x_i) + \mathcal{N}_k \\
    \end{array}
    \end{equation}

    Then, we show that 
    \begin{equation}\label{eq:mcp:vk}
        v^*(\mathcal{P}_{k,n}) \le \mathcal{V}_k
    \end{equation}
    By definition of $v^*(\mathcal{P}_{k,n})$ as the value of the maximum cut 
    in the residual graph we have: $v^*(\mathcal{P}_{k,n}) = 
    \max \left\{ \sum w_{ij} \mid k \le i < j < n, x_i \ne x_j \right\} $. From
    there, it immediately follows that:
    $$
    \begin{array}{rl}
    v^*(\mathcal{P}_{k,n}) &\le \mathcal{V}_k \\
    \max \left\{ \sum w_{ij} \mid k \le i < j < n, x_i \ne x_j \right\} 
    &\le \mathcal{V}_k \\
    \max \left\{ \sum w_{ij} \mid k \le i < j < n, x_i \ne x_j \right\} 
    &\le \sum_{k \le i < j < n} (w_{ij})^+ \\
    \end{array}
    $$
    
    Finally, we show that 
    \begin{equation}\label{eq:mcp:ski}
    \mathcal{C}_k \le \sum_{k \le i < n} |s^k_i|
    \end{equation}
    which is proved by contradiction. Indeed, falsifying that property would 
    require that $\mathcal{C}_k > \sum_{k \le i < n} |s^k_i|$. That is, it would 
    require that the value of the maximum cut between vertices in 
    $\mathcal{P}_{0,k}$ and those in $\mathcal{P}_{k,n}$ be strictly greater 
    than the sum of the marginal benefit for the "perfect" assignment of each 
    variable $\ge k$ wrt those $< k$. This either contradicts the definition of 
    a maximum (hence the definition of $\mathcal{C}_k$) or that of the state 
    transition function.
    
    Combining (\ref{eq:mcp:nk}), (\ref{eq:mcp:vk}) and (\ref{eq:mcp:ski}); we 
    have 
    $$
    %%\left.
    \begin{array}{c}
     v^*_{\mathcal{P}_{0,k}}(s^k) = v^*(s^k) - v_r + \mathcal{N}_k \wedge
     v^*(\mathcal{P}_{k,n}) \le \mathcal{V}_k                     \wedge
     \mathcal{C}_k \le \sum_{k \le i < n} |s^k_i| \\
    %%
    %%\right\} 
    \implies \\
    v^*_{\mathcal{P}_{0,k}}(s^k) + v^*(\mathcal{P}_{k,n}) + \mathcal{C}_k 
    \le 
    v^*(s^k) + \sum_{k \le i < n} |s^k_i| + \mathcal{V}_k + \mathcal{N}_k - v_r
    \end{array} 
    $$
    \qed
\end{proof}

\section{MAX2SAT}\label{appendix:rub:max2sat}
The DP model and relaxation we use for MAX2SAT again originates from 
\cite{bergman:16:branch-and-bound}. Their proof of correctness are to be found 
in the appendices of the same paper. Because the DP model and relaxation are 
extremely close for MAX2SAT and MCP, it was decided to omit their detailed 
formulation from these pages for the sake of conciseness.

\subsection{RUB}
In line with their DP models and relaxations, the RUB definition we used for 
MAX2SAT is very close to that of MCP. Formally, it is expressed as follows:
$$
\overline{s^k} = v^*(s^k) + \sum_{k \le i < n} |s^k_i| + \mathcal{V}_k + 
\mathcal{N}_k - v_r
$$
where $v_r = \sum_{0 \le i < n} w^{i,\neg{i}}$ is the sum of the weights of all 
tautological clauses in the problem\footnote{This is the only change we 
introduced to the original model, which one was not able to deal with 
tautological clauses.}. $\mathcal{N}_k$ is used to denote the sum of weights of 
tautological clauses in the subproblem bearing on assigned literals only. 
Formally, we have $\mathcal{N}_k = \sum_{0 \le i < k} w^{i,\neg{i}}$. Likewise, 
the definition of $\mathcal{V}_k$ is adapted to reflect the MAX2SAT nature of 
the problem. It is now defined as the sum of the maximal weights for the 
pairwise assignment of any two undecided variables. This is formally expressed 
as follows:
$$
\begin{array}{rcl}
\mathcal{V}_k &= 
& \sum_{k \le i < j < n} \max\left\{ w_{ij}^{TT}, w_{ij}^{TF}, w_{ij}^{FT}, 
w_{ij}^{FF} \right\} \\
&+ 
& \sum_{k \le i < n} \max\left\{ 
\left(w^{i, \neg i} + w^{i, i}\right), 
\left(w^{i, \neg i} + w^{\neg i, \neg i}\right) 
\right\} \\
\end{array}
$$ 
where $w_{ij}^{TT} = w^{i,j} + w^{i,\neg j} + w^{\neg i,j}$ is the total benefit 
reaped from assigning both variables $i$ and $j$ to true. Likewise, 
$w_{ij}^{TF}, w_{ij}^{FT}$ and $w_{ij}^{FF}$ are defined as follows:
$w_{ij}^{TF} = w^{i,j} + w^{i,\neg j} + w^{\neg i,\neg j}$ ; 
$w_{ij}^{FT} = w^{\neg i,j} + w^{\neg i,\neg j} + w^{i,j}$ ; 
$w_{ij}^{FF} = w^{\neg i,j} + w^{\neg i,\neg j} + w^{i,\neg j}$.

The proof of correctness for the MAX2SAT RUB is analogous to that of MCP.

\section{TSPTW}\label{appendix:rub:tsptw}
\paragraph{Important Note}
Because TSPTW is a \emph{minimization} problem, it does not respect the usual 
DDO maximization assumption. Still, DDO remains perfectly usable 
provided that relaxed MDD derives a \emph{lower bound} on the obective value and 
that RUB (which should rather be called \emph{rough lower bound} in this case) 
yields a lower bound on the optimal value rooted in any given node.

\subsection{DP Model}
The DP model we use to solve TSPTW with DDO is a variation of the classical DP 
formulation that minimizes the makespan. In essence, the model remains exactly 
the same, only did we adapt the definition of a state to make it more amenable 
to merging.

Formally, a state $s_k$ from the $k$th layer of the MDD is structured as a tuple 
$\left( position, earliest, latest, must\_visit, may\_visit \right)$ where 
$position$ is the set of cities where the traveling salesman might possibly be 
after he visited the $k$ first cities of his tour. As we will show later, this 
set is -- by definition of the transition function -- always a singleton except 
when $s_k$ is the state of a merged node. $Earliest$ denotes the earliest time 
when the salesman could have arrived to some city $ \in position$. Conversely, 
$latest$ denotes the latest moment at which the salesman might have arrived to 
some city $ \in position$. Again, in the absence of relaxed nodes, the time 
window $\left[earliest, latest\right]$would have a size of $1$ and thus denote 
an exact duration. Finally, $must\_visit$ is the set of cities that have 
not been visited on any of the $r-s_k$ path, and $may\_visit$ is the set of 
cities that have been visited along \emph{some} of the $r-s_k$ paths but 
\emph{not all of them}.

On that basis, the rest of the DP model is defined as follows:
\begin{itemize}
\item Root value: $v_r = 0$
\item Root state: $\left(\{0\}, 0, 0, \left\{1 \dots n-1\right\}, \emptyset 
\right)$
\item State transition function: 
$$
t\left(s_k, x_k\right) = \left(
t_{pos}\left(s_k, x_k\right), 
t_{earliest}\left(s_k, x_k\right), 
t_{latest}\left(s_k, x_k\right), 
t_{must}\left(s_k, x_k\right), 
t_{may}\left(s_k, x_k\right)
\right)
$$
where
$$
\begin{array}{ll}
t_{pos}\left(s_k, x_k\right)  &= \left\{x_k\right\} \\
t_{must}\left(s_k, x_k\right) &= s_k.must\_visit \setminus \left\{x_k\right\} \\
t_{may}\left(s_k, x_k\right)  &= s_k.may\_visit  \setminus \left\{x_k\right\} \\
t_{earliest}\left(s_k, x_k\right) &= 
\max\left(earliest(x_k) , \min\left\{ s_k.earliest + dist(p, x_k) \mid p \in 
s_k.position\right\}\right)\\
t_{latest}\left(s_k, x_k\right)&= 
\min\left(latest(x_k) , \max\left\{ s_k.latest + dist(p, x_k) \mid p \in 
s_k.position\right\}\right)\\
\end{array}
$$
\item Transition cost function: $h(s_k, x_k) = travel(s_k, x_k) + wait(s_k, 
x_k)$ where 
$$
\begin{array}{ll}
travel(x_k, s_k) &= \min\left\{dist(p, x_k) \mid p \in s_k.position \right\}\\
wait(x_k, s_k)   &= \left( earliest(x_k) - s_k.earliest + travel(s_k, x_k) 
\right)^+\\
\end{array}
$$
\end{itemize}

The state transition and transition cost functions $t(s_k, x_k)$ and $h(s_k, 
x_k)$ are only defined when the following three conditions hold:
\begin{itemize}
\item $x_k \in s_k.must\_visit \cup s_k.may\_visit$
\item $\min\left\{ s_k.earliest + dist(p, x_k) \mid p \in s_k.position\right\} 
\le latest(x_k)$
\item $\max\left\{ s_k.latest + dist(p, x_k) \mid p \in s_k.position\right\} \ge 
earliest(x_k)$
\end{itemize}

It is obvious that in the absence of relaxation, this model behaves as the 
classical DP model for TSPTW. Indeed, in that case $position$ is a always a 
singleton and therefore $earliest $ is always equal to $latest$. Also, one 
easily sees that $must\_visit$ complements the set of visited cities and 
$may\_visit$ is always empty. 

\subsection{Relaxation}
Our relaxation leverages the structure of the states to perform a merge operation
that loses as little information as possible. Our $\Gamma$ operator is defined 
as the identity function: the cost of edges entering the merged node don't need 
to be altered. Our $\oplus$ operator, on the other hand, is slightly more 
complex. It is defined as follows:
$$
\oplus(\mathcal{M}) =  \left(
\oplus_{pos}\left(\mathcal{M}\right), 
\oplus_{earliest}\left(\mathcal{M}\right), 
\oplus_{latest}\left(\mathcal{M}\right), 
\oplus_{must}\left(\mathcal{M}\right), 
\oplus_{may}\left(\mathcal{M}\right)
\right)
$$
where 
$$
\begin{array}{ll}
\oplus_{pos}\left(\mathcal{M}\right) &= 
    \bigcup_{u\in \mathcal{M}} u.position \\
\oplus_{earliest}\left(\mathcal{M}\right) &= 
    \min\left\{u.earliest \mid u\in \mathcal{M}\right\} \\
\oplus_{latest}\left(\mathcal{M}\right) &=
    \max\left\{u.latest \mid u\in \mathcal{M}\right\} \\
\oplus_{must}\left(\mathcal{M}\right) &= 
    \bigcap_{u\in \mathcal{M}} u.must\_visit \\
\oplus_{may}\left(\mathcal{M}\right) &=
    \left(\bigcup_{u\in \mathcal{M}} u.must\_visit \cup u.may\_visit \right) 
    \setminus
    \left(\bigcap_{u\in \mathcal{M}} u.must\_visit\right) \\
\end{array}
$$

\subsection{RLB (rough lower bound)}
First of all, the RLB we compute for TSPTW tries in a quick and inexpensive way 
to identify infeasible states: those which are guaranteed to enforce a 
constraint violation sooner or later. In those cases, the RLB returns the 
$\infty$ value. This approach closely relates to the fail first heuristic 
implemented in most constraint solvers.

We identified three infeasibility scenarios which can be efficiently checked:
\subsubsection{Too many unreachable cities}
When there are too many cities from the set $may\_visit$ which cannot be reached 
without violating the time window constraint (it becomes impossible to find a 
tour visiting $n$ cities without violating the time windows).
$$
\begin{array}{c}
\#\left\{ p \mid p \in s^k.may\_visit, s^k.earliest + shortest\_edge(p)
> latest(p) \right\} \\
> \\
n - k - \#s^k.must\_visit \\
\end{array}
$$

\subsubsection{One of the mandatory cities cannot be visited}
When one of the cities from the $must\_visit$ set cannot be possibly reached 
without violating its time window constraint.
$$
\exists_{p \in s^k.must\_visit} s^k.earliest + shortest\_edge(p) > latest(p)
$$

\subsubsection{When it is imposible to return to depot in time}
This occurs when the elapsed time since the start plus the sum of the shortest 
edges to each city $p$ makes it impossible to return to depot in due time.
Note, $shortest\_edge(p)$ is fixed throughout the whole optimization process as 
it denotes the shortest path to city $p$ from any other city. The sum $\sum_{p 
\in s^k.must\_visit} shortest\_edge(p)$ is thus an under approximation of the 
weight of an MST covering the remaining nodes to visit. The tightness loss 
which occurs vs using the actual weight of an MST is counter balanced by the 
tiny amount of time required to compute that estimate. This is a worthwhile 
trade off because of this RUB is computed at least once for every node in the 
MDD.
$$
s^k.earliest + \sum_{p \in s^k.must\_visit} shortest\_edge(p) > latest(depot)
$$

\subsubsection{Otherwise, when there are still cities to visit}
When $s^k.must\_visit \ne \emptyset$
$$
\begin{array}{rl}
s^k.earliest &+ \sum_{p \in s^k.must\_visit} shortest\_edge(p) \\
             &+ min\left\{ dist(p, depot) \mid p \in s^k.must\_visit \cup 
             s^k.may\_visit \right\} \\           
> & latest(depot)\\
\end{array}
$$

\subsubsection{Otherwise} When there are no cities left to visit
$$
s^k.earliest + dist(s^k.position, depot)
$$

\end{document}